\documentclass[letterpaper]{article} \usepackage[preprint]{aaai2027}
\usepackage[hyphens]{url}  \usepackage{graphicx}    \usepackage{natbib}  \usepackage{caption}   \usepackage{graphicx}
\usepackage{amsmath,amssymb}
\usepackage{amsthm}
\usepackage{thmtools}
\usepackage{enumitem}

\newtheorem{remark}{Remark}
\usepackage[disable]{todonotes}
\setuptodonotes{color=gray!30}
\usepackage{subcaption}
\usepackage[ruled,vlined,linesnumbered]{algorithm2e}
\usepackage{placeins}
\usepackage{booktabs}
\usepackage{tasks}
\usepackage{tabularx}
\usepackage{xspace}
\usepackage{csquotes}

\usepackage{tikz}

\usepackage{xcolor}
\newcommand{\rednote}[1]{}
\newcommand{\greennote}[1]{}

\usepackage{acro}
\DeclareAcronym{scm}{
    short = SCM,
    long = Structural Causal Model ,
    short-plural = s ,
    long-plural = s ,
}
\DeclareAcronym{nn}{
    short = NN,
    long = Neural Network ,
    short-plural = s ,
    long-plural = s ,
}
\DeclareAcronym{ibp}{
    short = IBP,
    long = Interval Bound Propagation ,
}
\DeclareAcronym{bab}{
    short = BaB,
    long = Branch-and-Bound ,
}

\copyrighttext{Preprint. Under review.}

\title{Computing Actual Causes for Neural Network Predictions\\ under Structured Causal Inputs}
\author{
    Jannick Strobel, Muqsit Azeem, Stefan Leue
}
\affiliations{
    University of Konstanz\\
    \{jannick.strobel, muqsit.azeem, stefan.leue\}@uni-konstanz.de 
}

\begin{document}

\maketitle

\newcommand{\causexbab}{$\textsc{CausEx}_{\textsc{BaB}}$\xspace}
\newcommand{\aci}{\textsc{ACI}\xspace}
\newcommand{\ilp}{\textsc{ILP}\xspace}
\newcommand{\brutef}{\textsc{BF}\xspace}

\newif\iflongrefs
\longrefsfalse 

\newcommand{\figname}{\iflongrefs Figure\else Fig.\fi}
\newcommand{\secname}{\iflongrefs Section\else Sec.\fi}
\newcommand{\appname}{\iflongrefs Appendix\else App.\fi}
\newcommand{\tabname}{\iflongrefs Table\else Tab.\fi}
\newcommand{\thmrefname}{\iflongrefs Theorem\else Thm.\fi} 
\newcommand{\remname}{\iflongrefs Remark\else Rem.\fi}
\newcommand{\algname}{\iflongrefs Algorithm\else Alg.\fi}

\newcommand{\figref}[1]{\figname~\ref{#1}}
\newcommand{\secref}[1]{\secname~\ref{#1}}
\newcommand{\appref}[1]{\appname~\ref{#1}}
\newcommand{\tabref}[1]{\tabname~\ref{#1}}
\newcommand{\thmref}[1]{\thmrefname~\ref{#1}}
\newcommand{\remref}[1]{\remname~\ref{#1}}
\newcommand{\algref}[1]{\algname~\ref{#1}}

\newcommand{\todojs}[2][]{\todo[color=cyan!40, #1]{\textbf{JS:} #2}}
\newcommand{\todoma}[2][]{\todo[color=orange!40, #1]{\textbf{MA:} #2}}
\newcommand{\todosl}[2][]{\todo[color=green!40, #1]{\textbf{SL:} #2}}

\begin{abstract}

 Explaining the predictions of neural networks is a central challenge in trustworthy AI.
 Existing explanation methods, such as those based on feature attribution or minimal sufficient sets, typically treat input features as independent, which can yield misleading explanations when inputs exhibit structured dependencies.
 We address this by formalizing explanations as Halpern--Pearl (HP) actual causes, modeling input dependencies using Boolean Structural Causal Models (SCMs).
 We compute HP causes by applying bound propagation and branch-and-bound techniques, while providing formal guarantees of completeness and minimality.
 Our experiments show that we substantially outperform brute-force and ILP baselines in scalability, and outperform heuristic search as graph size grows, computing all minimal actual causes on instances with search spaces of up to $2.3\times10^{13}$ candidate (cause, contingency) pairs, on SCMs with up to 28 nodes, within a 180s per-instance budget.
In a case study, we further show that ignoring input dependencies inflates the number of reported causes, 14.9\% of which are spurious under our SCM.

\end{abstract}
\section{Introduction}
\label{sec:intro}

Explaining the predictions of neural networks (NNs) is a central challenge in trustworthy AI~\cite{Zhang2021}. 
A useful explanation should identify why a particular prediction was made, and which feasible counterfactual changes would have altered it.
Many existing explanation methods rely on attribution scores that quantify the importance of input features, often under an implicit feature-independence assumption~\cite{Salih2025}.
However, in practice, inputs often exhibit structured causal dependencies: medical findings may constrain diagnoses, and high expenses may lead to negative cashflow while having an average income. 
Without accounting for these dependencies, a counterfactual explanation may vary negative cashflow directly instead of tracing it back to high expenses, so the root cause gets obscured among the other features it influences.
More generally, because changing one feature may constrain, determine, or invalidate changes to others, ignoring these dependencies can lead to misleading or even spurious explanations~\cite{Aas2021,Slack2020}.

This motivates explanations with explicit structural causal semantics. In order to represent this semantics, we adopt \emph{Halpern--Pearl (HP) actual causality}~\cite{Halpern2005,Halpern2015}. 
In this view, explaining a prediction means identifying minimal sets of features/variables whose values, if changed under an appropriate contingency, would change the prediction outcome.
To ensure that such counterfactual changes respect the structure of the input domain, we model feature dependencies using \emph{Structural Causal Models (SCMs)}~\cite{Pearl2009}. We assume that the SCM faithfully models the causal dependencies of the real-world domain, and thereby the dependencies present in the data on which the network is trained. Such SCMs can either be manually derived from domain knowledge, or learned from observational data via structure learning methods~\cite{Zheng2018}. We assume that such an SCM is given. \textbf{Our goal} is to compute HP actual causes, i.e., causal explanations, for a given NN prediction of an input whose feature dependencies are specified by a given SCM.

Importantly, such explanations need not be unique~\cite{Chockler2025}. Several minimal sets of variables may each be an actual cause of the same prediction. 
In a loan-risk model, for instance, one high-risk prediction might be explained by income alone, another by high expenses. 
Returning only one cause can hide these independent reasons, and for auditing and debugging we want the full set of minimal causal explanations, not just a single witness. In our experiments, 68.7\% of instances have more than one minimal cause.

Computing HP actual causes requires searching over candidate cause sets and contingencies via \emph{interventions}. An intervention sets a candidate cause to an alternative value, producing the counterfactual that shows whether the prediction would have changed. A contingency holds some of the other features fixed at their actual values, ruling out confounding from those features. This search is computationally hard in general~\cite{Eiter2006}.
Most existing work on actual-cause computation studies discrete causal or logical models without an NN predictor~\cite{Leitner-Fischer2013,Coenen2022,Ibrahim2020}.
Conversely, most NN explanation methods explain the predictor directly, without enforcing that counterfactual inputs respect an explicit SCM~\cite{Salih2025}.
The setting we study combines both: an SCM defines the structured input space, while an NN maps the resulting variables to a prediction.
This combination raises a difficult search problem, since one must reason simultaneously over candidate causes, contingencies, and the NN's prediction.

\textbf{Our key idea} is to compute HP actual causes by applying bound propagation and branch-and-bound techniques.
Instead of enumerating individual interventions and contingencies, we group them into regions of SCM-consistent inputs.
We then propagate sound bounds through the SCM and the NN via \ac{ibp}~\cite{Gowal2019}. During branch-and-bound~\cite{Bunel2020}, regions that provably preserve the prediction are pruned, while regions that provably change it witness actual causality. 
Inconclusive regions are refined by splitting and recursive invocation.

We focus on Boolean SCMs.
Their variables can represent categorical, thresholded, or logical properties, such as income status in a loan application or the presence of a clinical finding, while interventions and contingencies have a clear finite semantics.
This lets us isolate the core algorithmic problem of exact HP-cause computation.
At the same time, Boolean structural equations can be relaxed into multilinear functions that agree with the original SCM on all Boolean assignments.
As a result, we can reason soundly about whole regions of SCM-consistent 
inputs, rather than enumerating individual assignments one by one.

\noindent\textbf{Our contributions} can be summarized as follows:
\begin{itemize}
\item We formalize NN explanation under structured input dependencies as the computation of Halpern--Pearl actual causes over SCM-constrained inputs.
\item We introduce \causexbab, a sound and complete algorithm that
    combines multilinear relaxations of Boolean structural equations
    with bound propagation and branch-and-bound to return all minimal
    actual causes and minimum-cardinality witness contingencies.
    The search may be restricted to causes of size at most
    $k_{\max}$, preserving soundness and completeness within that bound.
\item We introduce synthetic SCM--NN benchmarks with trained NNs of varying capacity, and show that \causexbab substantially outperforms brute-force enumeration and ILP-based causality computation, and outperforms heuristic search as graph size grows, scaling to SCMs with up to 28 nodes and worst-case search spaces of up to $2.3\times10^{13}$ candidate cause--contingency pairs, and a 66.7\% majority at 30 nodes, with a 180s timeout.
\item In a case study for the U.S.\
Supplemental Nutrition Assistance Program (SNAP), \causexbab outperforms all baselines and shows that ignoring input dependencies
    more than doubles the median number of reported causes, 14.9\% of
    which are spurious under our SCM.
\end{itemize}

\section{Background}
\label{sec:background}
We give only a brief overview here. Please see \appref{app:prelims} and the cited references for a detailed background.

\paragraph{\aclp{scm}~\cite{Pearl2009}.}
We use \acp{scm} to represent causal relationships between variables. An \ac{scm} is a tuple $\mathcal{M} = (\vec{U}, \vec{V}, \mathcal{F})$ with exogenous variables $\vec{U} = \{U_1, \dots, U_d\}$, endogenous variables $\vec{V} = \{X_1, \dots, X_d\}$, and structural equations $\mathcal{F} = \{f_1, \dots, f_d\}$ such that $X_j := f_j(\mathrm{PA}_j, U_j)$, where $\mathrm{PA}_j \subseteq \vec{V} \setminus \{X_j\}$ are the endogenous parents of $X_j$. A \emph{context} $\vec{u}$ fixes $\vec{U}$ and thereby determines all of $\vec{V}$. We consider \emph{binary} SCMs, where $X_j \in \{0,1\}$ and each $f_j$ is a Boolean function. An \emph{intervention} $X \leftarrow x$ represents externally forcing $X$ to value $x$, overriding its usual causal mechanism, by replacing its structural equation with $X := x$.

\paragraph{Notation.}
We use $\vec V$, $\vec X$, and $\vec W$ to denote 
vectors of
variables. When combined with set-theoretic operators such as
$\subseteq$, $\setminus$, and $|\cdot|$, we identify these vectors, by
abuse of notation, with their underlying sets. In \secref{sec:approach}, where candidate causes and contingencies are manipulated
only as sets, we write $C$ and $W$ without an arrow.

\paragraph{Halpern and Pearl's Actual Causality~\cite{Halpern2015}.}
SCMs do not by themselves specify what counts as \emph{the} cause of an event. Halpern and Pearl's actual causality fills this gap and has gone through several revisions.
We adopt the latest, modified definition of actual causality~\cite{Halpern2015}, which avoids known problems with preemption and spurious causes at lower complexity. 
According to this definition, given $(\mathcal{M}, \vec{u})$, a set of variables $\vec{X} = \vec{x}$ is an \emph{actual cause} of an outcome $\varphi$ if
(\textbf{AC1}) $(\mathcal{M}, \vec{u}) \models \vec{X} = \vec{x}$ and $(\mathcal{M}, \vec{u}) \models \varphi$,
(\textbf{AC2}) there exist $\vec{W} \subseteq \vec{V} \setminus \vec{X}$ with actual value $\vec{w}$ and $\vec{x}' \neq \vec{x}$ such that $(\mathcal{M}, \vec{u}) \models [\vec{X} \leftarrow \vec{x}',\, \vec{W} \leftarrow \vec{w}]\, \lnot\varphi$, and
(\textbf{AC3}) $\vec{X}$ is minimal, i.e., no strict subset of $\vec{X}$ satisfies AC1 and AC2.
Intuitively, AC1 requires that the cause and the outcome both hold in the actual context, AC2 requires that changing $\vec{X}$ under some contingency $\vec{W}$ leads to $\neg\varphi$, and AC3 requires that no part of $\vec{X}$ alone would suffice. A \emph{contingency} $\vec{W}$ is a set of other variables held fixed at their actual values $\vec{w}$ while $\vec{X}$ is changed, so other variables that also affect $\varphi$ cannot mask or compensate for the change.

\paragraph{\acp{nn}.}
We consider \acp{nn} that map a vector of Boolean variables to a continuous value, formally represented as the mapping $N: \{0,1\}^d \rightarrow \mathbb{R}$, $N(\vec{x}) = \sigma_L \circ \sigma_{L-1} \circ \cdots \circ \sigma_1(\vec{x})$
where $\sigma_i(z) = \max(0, w_i \cdot z + b_i)$ is a ReLU-activated linear layer.

 \paragraph{Branch-and-Bound based \acl{nn} Analysis.}
 \ac{nn} verification decides whether a property of $N$'s output holds over an input region: given bounds on the inputs, it computes sound bounds on $N$'s output. \ac{ibp}~\cite{Gowal2019} propagates an input box through each layer using interval arithmetic. When bounds are inconclusive, \ac{bab}~\cite{Bunel2020} splits the region into sub-regions and recomputes bounds on each, until the property is certified, refuted, or all sub-regions are explored.

\section{Approach}
\label{sec:approach}

\definecolor{treeBaseFill}{HTML}{f7cfe2}
\definecolor{treeBaseEdge}{HTML}{DC267F}
\definecolor{treeSkipFill}{HTML}{CCCCCC}
\definecolor{treeSkipEdge}{HTML}{888888}
\definecolor{treeBranchFill}{HTML}{D0E8FF}
\definecolor{treeBranchEdge}{HTML}{1F77B4}
\definecolor{treeConfirmedFill}{HTML}{CCF0CC}
\definecolor{treeConfirmedEdge}{HTML}{1A6E1A}
\definecolor{treePrunedFill}{HTML}{FFCCCC}
\definecolor{treePrunedEdge}{HTML}{CC2020}
\definecolor{treeNaturalEdge}{HTML}{648FFF}
\definecolor{treeContingencyEdge}{HTML}{E68A00}
\definecolor{undecided}{HTML}{5C5C5C}

We compute actual causes of \ac{nn} predictions under structured input dependencies, modeled by an SCM $\mathcal{M}$, under the modified HP definition of actual causality.
We present our approach through a running example. We first show how
exhaustive actual-cause computation searches over candidate causes and
contingencies, and then explain how \causexbab replaces this enumeration
with bound propagation and branch-and-bound over regions.

\paragraph{Problem Setup.}
The context $\vec u$ induces the actual assignment $\vec x$ and prediction $y = N(\vec x)$ (\secref{sec:background}). We say an assignment $\vec x'$ satisfies the outcome $\varphi$ if $N(\vec x') \in S$, for some fixed \emph{acceptance set} $S \subseteq \mathbb{R}$ that specifies what counts as the outcome. This set is a modeling choice, orthogonal to the rest of our approach. Throughout this section, including the running example, we take $S = [\tau, \infty)$, where $\tau$ is a hard threshold on the output. \secref{sec:exp} instead specifies $\varphi$ via a distance metric. $\varphi$ holds when $N(\vec x')$ lies within a threshold distance of the actual output, i.e., $S$ becomes an interval around it. Our goal is to find all sets of variables $\vec X \subseteq \vec V$ that are \emph{actual causes} (AC1--AC3) of $\varphi$, witnessed by an intervention $\vec X \leftarrow \vec x'$ that, under some contingency, falsifies $\varphi$.

\subsection{Running Example and Exhaustive Search}\label{subsec:runningexample}
We illustrate our algorithm using a concrete example.
Consider a bank that uses an \ac{nn} to predict the default risk of loan applicants.
Given information about an applicant, the \ac{nn} outputs a risk score
$y \in [0,1]$ indicating the probability of loan default.
The bank classifies an applicant as \emph{high risk} if $y \in S$, for $S = [\tau, \infty)$, with $\tau = 0.5$, and as \emph{low risk} otherwise.
We model five binary applicant features:
$X_1$~(high income), $X_2$~(negative cashflow), $X_3$~(high expenses),
$X_4$~(unsecured exposure), and $X_5$~(has guarantor).
These features are causally related: high income can cause high expenses;
high expenses combined with no high income can cause negative cashflow;
negative cashflow without a guarantor indicates unsecured exposure.

We formalize this as a Structural Causal Model.
\figref{fig:loan_scm} illustrates the SCM with its structural equations together with the \ac{nn}
classifier $N$.
Each endogenous variable $X_i$ is determined by its causal parents and an
exogenous variable $U_i$ via a structural equation.
For root nodes (without parents), $X_i = U_i$.
For non-root nodes, $X_i$ is a node-specific Boolean function of its parents, combined with $U_i$ via XOR. Concretely, $X_3 = X_1 \oplus U_3$, $X_2 = (X_3 \wedge \lnot X_1) \oplus U_2$, and $X_4 = (X_2 \wedge \lnot X_5) \oplus U_4$.
The exogenous variables $U_1, \ldots, U_5$ encode all background factors not
explicitly modeled.
Their role is to capture uncertainty: even with high income, high expenses
are not guaranteed, and negative cashflow without a guarantor does not
necessarily lead to unsecured exposure.
Given a fixed context $\vec{u} = (u_1, \ldots, u_5)$, all endogenous
features are fully determined, and the predictor output $y = N(X_1,
\ldots, X_5)$ can be computed.

\begin{figure}[t]
\centering
\includegraphics[width=1.0\columnwidth]{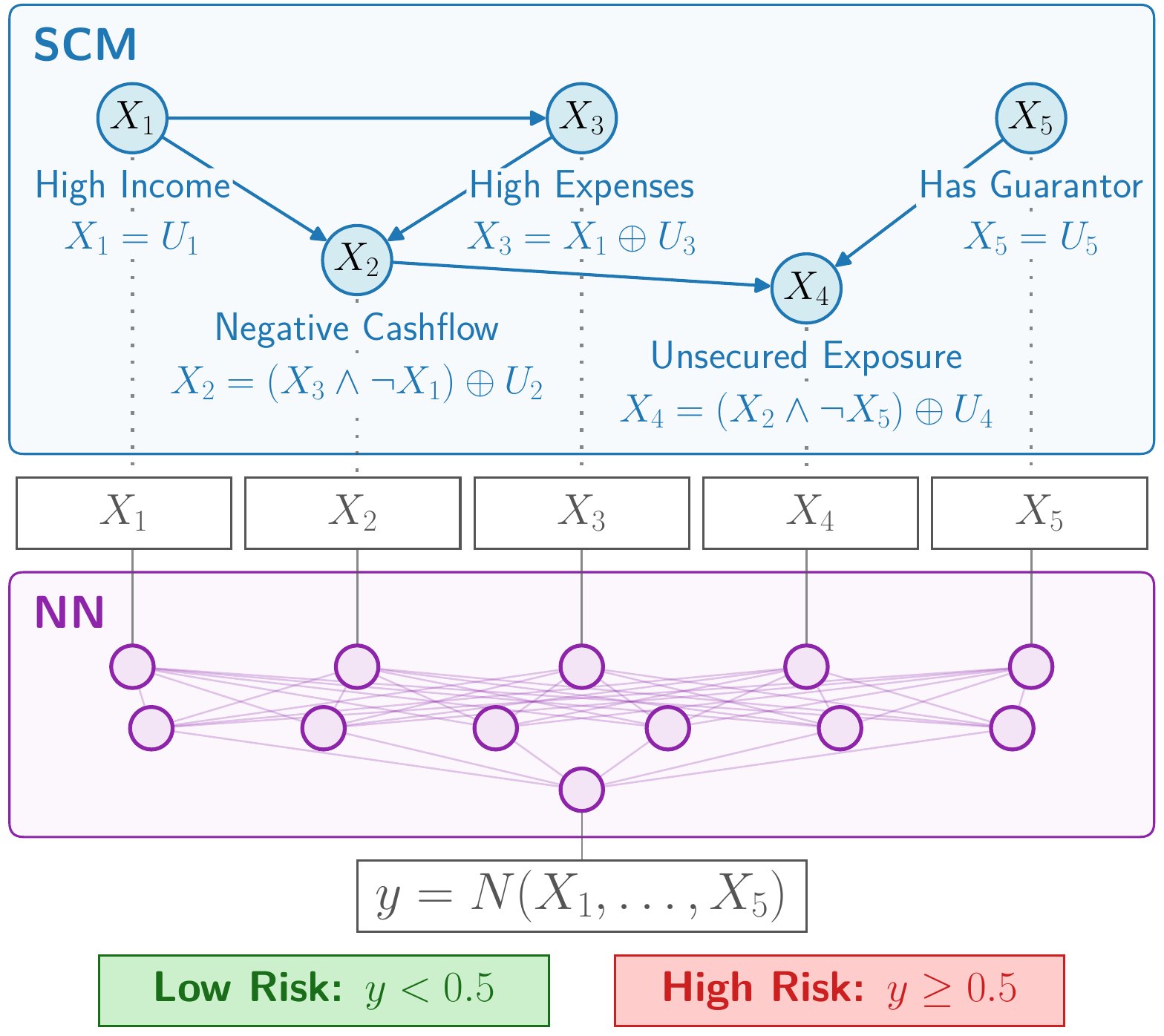}
\caption{Loan risk SCM, as presented in \secref{subsec:runningexample}.}
    \label{fig:loan_scm}
\end{figure}

We fix the following context for this  example throughout the paper: $\vec u = (1,1,0,1,0)$. This induces the actual assignment $\vec x = (1,1,1,0,0)$\footnote{\scriptsize Computed in topological order: $X_1{=}1$, $X_3{=}1$, $X_5{=}0$, $X_2{=}1$, $X_4{=}0$.}: the applicant has high income, negative cashflow, high expenses, no unsecured exposure, and no guarantor. The predictor outputs $y = N(\vec x) = 0.66 \ge \tau$, so the bank classifies the applicant as \emph{high risk}. Our goal is to find minimal sets of features $\vec X \subseteq \{X_1, \ldots, X_5\}$ that are actual causes of this classification, that is, interventions that, under some contingency, would push the prediction below $\tau$ and flip the applicant to \emph{low risk}.

\smallskip
\noindent\textbf{Exhaustive search}
enumerates every candidate cause set $C \subseteq \vec V$ assigning its
variables the \textbf{\textcolor{treeBaseEdge}{\textsc{cause}}} role and fixing them to their counterfactual values $\vec x'_C$ (i.e.\ the Boolean complements of their actual values).
Checking AC2 for $C$ then amounts to determining
whether some contingency makes the intervention
$C \leftarrow \vec x'_C$ falsify $\varphi$.
For this, we enumerate every subset
$W \subseteq \vec V \setminus C$.
Variables in $W$ take the
\textbf{\textcolor{treeContingencyEdge}{\textsc{contingency}}} role and are fixed to their actual values $\vec w$,
while the remaining variables
$\vec V \setminus (C \cup W)$ take the \textbf{\textcolor{treeNaturalEdge}{\textsc{natural}}} role and
follow their structural equations.
Each choice of $W$ thus
induces one fully-determined counterfactual assignment $\tilde{\vec x}$
and prediction $N(\tilde{\vec x})$. If
$N(\tilde{\vec x}) \notin S$ for some $W$, then $C$ satisfies AC2,
witnessed by $W$. Otherwise, $C$ is not a cause.

To enforce minimality (AC3), we enumerate candidate cause sets by increasing size, discarding any set that is a superset of an already-found cause. For each remaining candidate, check contingencies by increasing size, so the first one falsifying $\varphi$ is a minimum-cardinality witness. Both the number of candidate cause sets and the number of contingencies per candidate grow exponentially in $|\vec V|$, so this exhaustive approach quickly becomes infeasible.

\begin{figure*}[t]
\centering
\includegraphics[width=\textwidth]{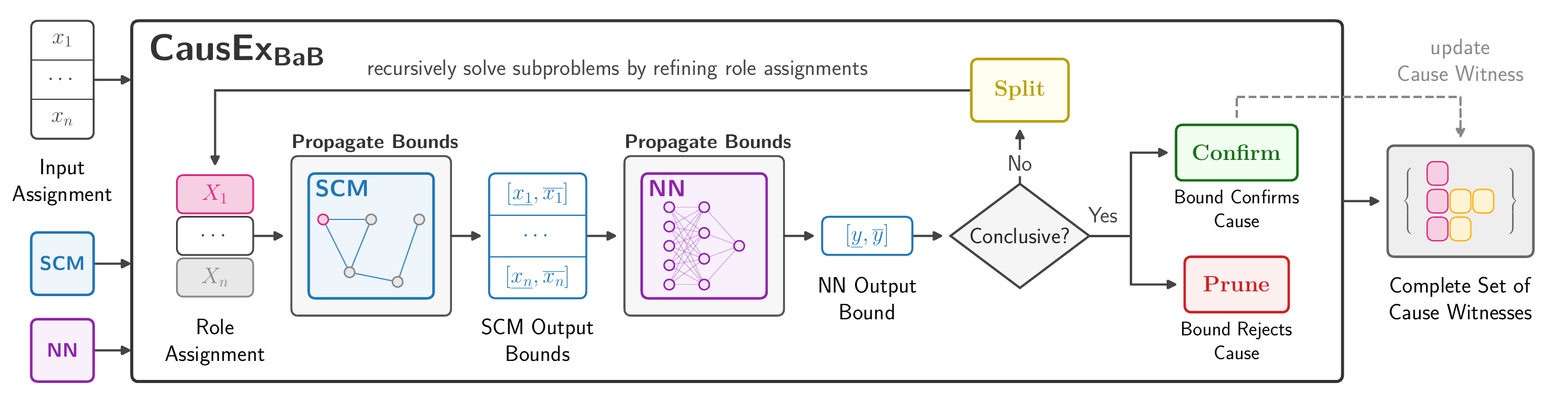}
\caption{Overview of \causexbab. For a selected candidate cause, an input 
together with a role assignment (\textcolor{treeBaseEdge}{\textsc{cause}} / \textcolor{treeContingencyEdge}{\textsc{contingency}} / \textcolor{treeNaturalEdge}{\textsc{natural}} / \textcolor{undecided}{\textsc{undecided}}) is propagated first through the SCM via interval arithmetic and then through the \ac{nn} via \ac{ibp}, to obtain an NN output bound. If this bound is conclusive, the region is pruned or confirmed, and a confirmed region yields a cause witness; otherwise, the role assignment is split into subproblems, and the process repeats recursively on each subproblem.}\label{fig:overview}
\end{figure*}

\subsection{The \causexbab Algorithm}
Rather than checking candidate cause and contingency pairs one at a time, we reason about \emph{regions} that represent many possible role assignments at once. 
To this end, we mark a variable \textcolor{undecided}{\textbf{\textsc{undecided}}} when it has not been resolved to \textsc{natural}, in which case it follows its structural equation, or to \textsc{contingency}, in which case it is fixed to its actual value.
\figref{fig:overview} presents an overview of our approach, which we refer to as \causexbab.
We analyze each region using sound output bounds. Conclusive regions are pruned or confirmed without examining their 
assignments individually, while inconclusive regions are refined by resolving one \textsc{undecided} variable and recursively analyzing the two resulting subregions.

\paragraph{Representing Contingencies as Regions.}
For a fixed candidate cause set $C$, assigning every remaining variable either the \textsc{natural} or \textsc{contingency} role specifies one possible contingency. 
\causexbab compactly represents many such assignments using \emph{partial role assignments}, in which variables may remain \textsc{undecided}. 
A region $\mathcal{B}$ consists of a partial role assignment together with an interval $[\ell_X,h_X]$ for each variable $X\in\vec V$, representing the values that $X$ may take across the contingencies captured by the region.
For \textsc{cause} and \textsc{contingency} variables, this interval collapses to the corresponding fixed value. 
For \textsc{natural} variables, it is induced by the structural equation, while for \textsc{undecided} variables it is the hull of the \textsc{natural} and \textsc{contingency} possibilities.

We write $\mathcal{B}_0(C)$ for the initial region of candidate cause set $C$, in which the variables in $C$ take the \textsc{cause} role and all remaining variables are \textsc{undecided}. 
This region represents all possible contingencies for $C$. 
We next explain how its intervals are computed and propagated through the SCM.

\paragraph{Propagating Bounds.}
To propagate intervals through the Boolean SCM, we use the continuous relaxations of Boolean functions as used in \citet{Petersen2022}.

\begin{tabular}{@{} l l @{}}
    \textbullet~ $\text{NOT}(a) = 1-a$ \quad & \textbullet~ $\text{OR}(a,b) = a+b-ab$ \\[2pt]
    \textbullet~ $\text{AND}(a,b) = ab$ \quad & \textbullet~ $\text{XOR}(a,b) = a+b-2ab$ \\
\end{tabular}

These relaxations are exact on Boolean inputs.
Each relaxation is multilinear. 
As a result, the extrema of a single relaxed gate over its parents' interval bounds are always attained at a corner of the corresponding interval box. This gives the tightest possible interval for that gate alone.
On $[0,1]$, $\text{NOT}$, $\text{AND}$, and $\text{OR}$ are additionally monotonic in each argument ($\text{NOT}$ decreasing, and $\text{AND}$ and $\text{OR}$ increasing). For these gates, evaluating the appropriate interval endpoints suffices, whereas the non-monotonic $\text{XOR}$ gate requires evaluating all four corners.
The resulting map $\mathrm{bound}_{\mathrm{SCM}}$ processes the variables of $\mathcal M$ in topological order, applying the appropriate role rule using each variable's already-computed parent bounds.
$\mathrm{bound}_{\mathrm{NN}}$ then propagates the resulting feature bounds through $N$ via \ac{ibp}~\cite{Gowal2019}. 
Affine layers are bounded using interval arithmetic, and each ReLU maps an interval $[\ell,h]$ to $[\max(0,\ell),\max(0,h)]$.

\paragraph{\causexbab.}
Our algorithm, presented in \algref{alg:causexbab}, computes minimal causes with minimum-cardinality contingencies that satisfy AC1--AC3.\footnote{
\algref{alg:causexbab} searches all candidate sizes
$k=1,\ldots,|\vec V|$. The outer loop may be capped at $k_{\max}<|\vec V|$ to reduce computation, 
in which case 
completeness holds only for causes of size at most $k_{\max}$.}
AC1 is satisfied by assumption. The search only considers actual instances where $\varphi$ already holds ($y \in S$). The algorithm establishes AC2 and enforces AC3 by pruning or confirming regions based on their bounds $[\ell,h]$ on $N(\vec x')$. A region can be pruned if $[\ell,h] \subseteq S$ (every input satisfies $\varphi$) and confirmed if $[\ell,h] \cap S = \emptyset$ (every input falsifies $\varphi$). For our threshold instantiation $S = [\tau,\infty)$, this specializes to $\ell \geq \tau$ and $h < \tau$ respectively. \causexbab encompasses the following steps:
{\footnotesize
\begin{algorithm}[htb]
\caption{\causexbab}
\label{alg:causexbab}
\DontPrintSemicolon
\KwIn{$\vec x$ with $N(\vec x)\in S$, SCM $\mathcal M$, NN $N$, set $S$}
\KwOut{Set of all $(C, W^*)$ pairs} 

$\mathcal{C}^* \leftarrow \emptyset$\;

\For{$k = 1, \dots, |\vec V|$}{
  \For{each $k$-subset $C \subseteq \vec V$}{
    \lIf{$\exists\, (C', W') \in \mathcal{C}^*$ s.t. $C' \subsetneq C$}{\textbf{skip}}
    $Q \leftarrow \{\mathcal{B}_0(C)\}$\;
    $W^* \leftarrow \bot$
    \tcp{best cont. for $C$ so far, $\bot$ = none found yet}

    \While{$Q \neq \emptyset$}{
      $\mathcal{B} \leftarrow \mathrm{pop}(Q)$\;
      $W \leftarrow \mathrm{contingency}(\mathcal{B})$\;
      \tcp{get cont. vars. from $\mathcal{B}$}
      \lIf{$W^* \neq \bot$ and $|W| \geq |W^*|$}{\textbf{prune}}
      \Else{
      $[\vec\ell, \vec h] \leftarrow \mathrm{bound}_{\mathrm{SCM}}(\mathcal{M}, \mathcal{B})$\;
      $[\ell, h] \leftarrow \mathrm{bound}_{\mathrm{NN}}(N,\, [\vec\ell, \vec h])$\;
\lIf{$[\ell, h] \subseteq S$}{\textbf{prune}}
\uElseIf{$[\ell, h] \cap S = \emptyset$}{
\lIf{$W^* = \bot$ or $|W| < |W^*|$}{$W^* \leftarrow W$}
      }
      \Else{
        $j \leftarrow \mathrm{selectVar}(\mathcal{B})$\;
        $Q \leftarrow Q \cup \{\mathcal{B}[j : \textsc{natural}]\}$\;
        $Q \leftarrow Q \cup \{\mathcal{B}[j : \textsc{contingency}]\}$
        \tcp*{branch selected node and add to queue}
      }
      }
    }
    \lIf{$W^* \neq \bot$}{$\mathcal{C}^* \leftarrow \mathcal{C}^* \cup \{(C, W^*)\}$ \tcp*[h]{witness found: record $C$}
    }
  }
}
\Return $\mathcal{C}^*$\;
\end{algorithm}}

\begin{enumerate}
	\item \textbf{Enumerate candidates.} We consider candidate cause sets $C \subseteq \vec V$ by increasing size $k = 1,\ldots, |\vec V|$. Before searching a candidate, we check whether some already-confirmed cause $C'$ is a strict subset of $C$. If so, $C$ cannot be minimal, and we skip it without further computation. Because candidates are visited in increasing order of size, this check guarantees AC3
    minimality of every cause set we report. For a surviving candidate $C$, we initialize a queue $Q$ with the single region $\mathcal{B}_0(C)$, in which the variables in $C$ take the \textsc{cause} role and every other variable is \textsc{undecided}. This region represents all possible contingencies
	for the intervention on $C$.

	\item \textbf{Compute bounds and classify.} We repeatedly pop a region $\mathcal{B}$ from $Q$. Let $W=\mathrm{contingency}(\mathcal{B})$ be the variables already assigned the \textsc{contingency} role. If $W^*\neq\bot$ and $|W|\geq|W^*|$, no refinement of $\mathcal{B}$ can yield a smaller contingency, so we prune it.
Otherwise, $\mathrm{bound}_{\mathrm{SCM}}$ propagates the intervals induced by $\mathcal{B}$ through $\mathcal{M}$ to obtain bounds
    $[\vec\ell, \vec h]$ on every endogenous variable.We then propagate these feature bounds through $N$ via $\mathrm{bound}_{\mathrm{NN}}$ to obtain bounds $[\ell, h]$ on $N$ for all assignments represented by $\mathcal{B}$.
    If $[\ell,h] \subseteq S$, then every input in $\mathcal{B}$ still satisfies $\varphi$, so no contingency in this region can falsify it, and we \emph{prune} $\mathcal{B}$. If $[\ell,h] \cap S = \emptyset$, every input represented by $\mathcal{B}$ falsifies $\varphi$, so we \emph{confirm} $\mathcal{B}$. 
    Resolving all remaining \textsc{undecided} variables to \textsc{natural} yields the smallest witness contingency represented by $\mathcal{B}$, which we retain only if it is smaller than the best one found for $C$ so far.
Otherwise, the bounds are inconclusive, and we \emph{refine} $\mathcal{B}$.

	\item \textbf{Refine.} To refine an inconclusive region, we select one \textsc{undecided} variable using a branching heuristic (see \secref{sec:exp}, \appref{app:branchselection}), and split $\mathcal{B}$ into two children: one where that variable is resolved to \textsc{natural}, and one where it is resolved to \textsc{contingency}. Both children are added back to $Q$. Once $Q$ is empty, if a witness contingency was found for $C$, we add the pair $(C,W^*)$ to $\mathcal{C}^*$.
\end{enumerate}

This yields two guarantees: (i) candidate ordering enforces AC3, and
(ii) continuing the search after the first witness contingency while pruning regions
whose committed contingency is no smaller than the current best witness
yields a minimum-cardinality contingency
(see \remref{rem:responsibility} in \appref{app:approach}).

\begin{remark}[Restricted intervention set]
For clarity, \algref{alg:causexbab} assumes that all endogenous variables
may appear in candidate causes. In practice, cause interventions may
instead be restricted to a user-specified set
$I\subseteq\vec V$, for example by excluding administrative
variables. Candidate causes are then drawn from $I$, while the remaining
variables, whether inside or outside $I$, may still serve as contingency
variables. \end{remark}

\newcommand{\treescale}{0.57}
\newcommand{\treefontbase}{\large}
\newcommand{\treefontnode}{\normalsize}
\newcommand{\treefontlabel}{\normalsize}

\tikzset{
  root/.style={draw, circle, fill=black, minimum size=2mm, inner sep=0pt},
  base/.style={draw=treeBaseEdge, rounded corners, line width=1.0pt, fill=treeBaseFill, font=\treefontbase},
  skip/.style={draw=treeSkipEdge, rounded corners, fill=treeSkipFill, font=\treefontbase, dashed},
  branch/.style={draw=treeBranchEdge, line width=1.0pt, circle, fill=treeBranchFill, font=\treefontnode, inner sep=1pt},
  confirmed/.style={draw=treeConfirmedEdge, line width=1.0pt, circle, fill=treeConfirmedFill, font=\treefontnode},
  pruned/.style={draw=treePrunedEdge, line width=1.0pt, circle, fill=treePrunedFill, font=\treefontnode},
  edge from parent/.style={draw, line width=1.8pt, -latex},
  every node/.style={align=center}
}

\paragraph{Algorithm illustration on the running example in \figref{fig:trees}.} Recall $S = [\tau, \infty)$ with $\tau = 0.5$ for this example.
Initially, $X_1$ is fixed to \textcolor{treeBaseEdge}{\textsc{cause}} and every other variable is left \textcolor{undecided}{\textsc{undecided}}. The initial bounds $[0.08, 1.00]$ are inconclusive (neither $[\ell,h]\subseteq S$ nor $[\ell,h]\cap S=\emptyset$), so the search branches on $X_2$, and both children branch again on $X_3$ for the same reason. All four resulting leaves have bounds entirely in $S$ and satisfy $\varphi$, so $\{X_1\}$ is pruned entirely.
For $\{X_3\}$ and $\{X_4\}$, fixing the singleton leaves nothing else undetermined, so both bounds resolve to a single point without branching: $0.41$ (outside $S$) confirms $\{X_3\}$ as a cause, while $\{X_4\}$'s point lies in $S$ and is pruned. $\{X_5\}$ (bounds $[0.35, 0.83]$) and $\{X_2\}$ are both inconclusive 
and branch once on $X_4$: for $\{X_5\}$ both children's bounds lie in $S$ and are pruned, while for $\{X_2\}$ the \textcolor{treeNaturalEdge}{\textsc{natural}} child is pruned but the \textcolor{treeContingencyEdge}{\textsc{contingency}} child's bounds lie outside $S$, falsifying $\varphi$ and confirming $\{X_2\}$ as a second cause with witness contingency $\{X_4\}$. Concretely, intervening on negative cashflow alone does not flip the classification unless unsecured exposure is also 
held
at its actual value $X_4{=}0$, the contingency. This leaves $\{X_2\}$ and $\{X_3\}$ as confirmed causes of size $k=1$.

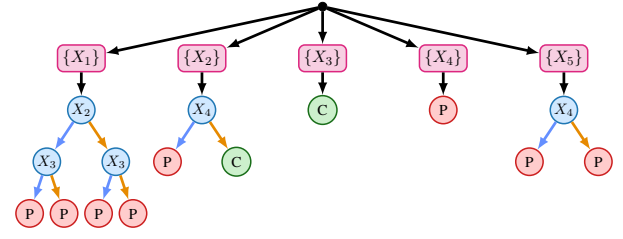
\begin{figure}[t]
\centering
\scalebox{\treescale}{
\begin{tikzpicture}[
  level distance=12mm,
  level 1/.style={sibling distance=28mm},
  level 2/.style={sibling distance=14mm},
  level 3/.style={sibling distance=16mm},
  level 4/.style={sibling distance=8mm},
]
\node[root] {}
  child { node[base] {$\{X_1\}$}
    child { node[branch] {$X_2$}
      child { node[branch] {$X_3$} edge from parent[draw=treeNaturalEdge]
        child { node[pruned] {P} edge from parent[draw=treeNaturalEdge] }
        child { node[pruned] {P} edge from parent[draw=treeContingencyEdge] }
      }
      child { node[branch] {$X_3$} edge from parent[draw=treeContingencyEdge]
        child { node[pruned] {P} edge from parent[draw=treeNaturalEdge] }
        child { node[pruned] {P} edge from parent[draw=treeContingencyEdge] }
      }
    }
  }
  child { node[base] {$\{X_2\}$}
    child { node[branch] {$X_4$}
      child { node[pruned] {P} edge from parent[draw=treeNaturalEdge] }
      child { node[confirmed] {C} edge from parent[draw=treeContingencyEdge] }
    }
  }
  child { node[base] {$\{X_3\}$}
    child { node[confirmed] {C} }
  }
  child { node[base] {$\{X_4\}$}
    child { node[pruned] {P} }
  }
  child { node[base] {$\{X_5\}$}
    child { node[branch] {$X_4$}
      child { node[pruned] {P} edge from parent[draw=treeNaturalEdge] }
      child { node[pruned] {P} edge from parent[draw=treeContingencyEdge] }
    }
  };
\end{tikzpicture}
}
\caption{\causexbab search tree for $k=1$ on our running example. Each root child (e.g.\ $\{X_1\}$) is a candidate \textcolor{treeBaseEdge}{\textsc{cause}} (pink). At each branch, the left child resolves the branching variable to \textcolor{treeNaturalEdge}{\textsc{natural}} (blue) and the right child to \textcolor{treeContingencyEdge}{\textsc{contingency}} (orange). \textcolor{treeConfirmedEdge}{C = confirmed}, \textcolor{treePrunedEdge}{P = pruned}.}
\label{fig:trees}
\end{figure}

Every candidate of size $k=2,\ldots,5$ is either a superset of $\{X_2\}$ or $\{X_3\}$ and skipped, or checked and pruned, so no cause larger than $k=1$ exists for our running example: $\{X_3\}$ (\emph{high expenses}) and $\{X_2\}$ (\emph{negative cashflow}) with witness contingency $\{X_4\}$ (\emph{unsecured exposure}) are the only two minimal causes. The minimal sufficient set for this instance is the single set $\{X_2, X_3\}$, exactly the union of the two causes. Computing actual causes is more informative here: it identifies $X_2$ and $X_3$ as independently responsible, each with its own witness contingency, rather than only the combined sufficient set, which does not reveal that either feature alone already suffices to flip the classification. 

\appref{app:running_example_orig} works through a different context on the same SCM that does produce a cause of size $k=2$.

\paragraph{Correctness and Termination.}
\begin{restatable}{theorem}{thmcorrectness}
\label{thm:correctness}
Given an SCM $\mathcal{M}$, NN $N$, and $N$'s actual input $\vec x$ with $N(\vec x) \in S$, \algref{alg:causexbab} terminates. It returns exactly the set of all actual causes of $\varphi$ (characterized by AC1--AC3).\footnote{If the maximum cause size is capped at $k_{\max} < |\vec V|$, this guarantee holds only for actual causes with $|C| \leq k_{\max}$. See \appref{app:proofs}, \remref{rem:kmax}.}
For each returned cause, it also returns a minimum-cardinality contingency.
\end{restatable}
\noindent Please refer to the \appref{app:proofs} for the full proof of \thmref{thm:correctness}.

\paragraph{Complexity.}
For a candidate cause $C$ with $|C|=k$, the search resolves
$|\vec V|-k$ non-cause variables. In the worst case, the resulting
binary search tree contains
$\sum_{i=0}^{|\vec V|-k}2^i
=2^{|\vec V|-k+1}-1$
regions, including the initial region. Hence, the overall number of
regions is at most
$\sum_{k=1}^{|\vec V|}
\binom{|\vec V|}{k}
\left(2^{|\vec V|-k+1}-1\right)
<2\cdot 3^{|\vec V|}$,
which is $O(3^{|\vec V|})$. See \appref{app:complexity} for details.

\section{Experiments}
\label{sec:exp}

We evaluate \causexbab on both synthetic benchmarks and a
case study. The synthetic benchmarks allow us to systematically vary graph
size, causal structure, maximum cause size, and NN capacity,
while the case study demonstrates applicability to a real decision-making
setting.
Because established benchmarks for HP-cause computation in NNs
with structured causal inputs are limited, we generate the synthetic
benchmarks described below.

\subsection{Experimental Setup}

We ran all experiments on an NVIDIA L40S GPU, an AMD EPYC 9375F CPU (32 cores), 192\,GB RAM, and Ubuntu 24.04.4 LTS. We compare \causexbab against brute-force enumeration (\brutef), a search-based baseline (\aci)~\cite{Reyd2025}, and a constraint-based (\ilp) formulation~\cite{Ibrahim2020}
, encoding ReLU via Big-M constraints 
and solved with SCIP~\cite{Hojny2025}, using a \textbf{180s timeout} per run.
For fairness, \brutef includes the same superset pruning, batched GPU evaluation, and JIT-compiled SCM/NN evaluation used by \causexbab. Unless stated otherwise, \causexbab uses the \texttt{width\_large} branching heuristic; \appref{app:branchselection} reports on alternative heuristics.

\paragraph{SCM Generation.}
Following the benchmark protocol of \citet{Aryan2026}, we generate
Barab\'{a}si--Albert (BA)-style preferential-attachment
DAGs~\cite{Barabasi1999} by adding nodes sequentially and allowing each
new node to select only earlier nodes as parents. This yields a small
number of hub variables and many sparsely connected variables.
Each endogenous variable is assigned an AND or OR structural equation with equal probability ($p=0.5$), and each incoming edge is independently negated with probability $p=0.2$.
We also run the same experiments on DAGs generated using the
Erd\H{o}s--R\'{e}nyi (ER) model~\cite{Erdos1959} and observe consistent
runtime trends (\appref{app:er_results}).

\paragraph{Neural Network Generation.}
To obtain a nontrivial task over SCM-generated inputs, we use a
frozen, randomly initialized multilayer perceptron (MLP) with two hidden
layers as a teacher network~\citep{Goldt2020}.
We normalize its output to $[0,1]$, add Gaussian
noise, and train a student NN to approximate it. 
Given the
actual output $y$, we set the acceptance set to
$S=[y-\varepsilon,y+\varepsilon]$, with $\varepsilon=0.3$. Thus, an
intervention yielding output $y'$ falsifies $\varphi$ if
$|y'-y|>\varepsilon$.

Each configuration is evaluated on three independently generated graphs
and three sampled contexts per graph, yielding nine runs per configuration.

\paragraph{Validation against exhaustive search.}
On every instance completed by both methods, \causexbab and exhaustive
\brutef return exactly the same set of minimal causes.

\subsection{Results}
\label{sec:results}
We first evaluate how \causexbab scales with the number of SCM nodes
relative to \brutef, \aci, and \ilp, before reporting a series of
additional experiments.

\begin{figure}[t]
\centering
\resizebox{1.0\linewidth}{!}{
\includegraphics{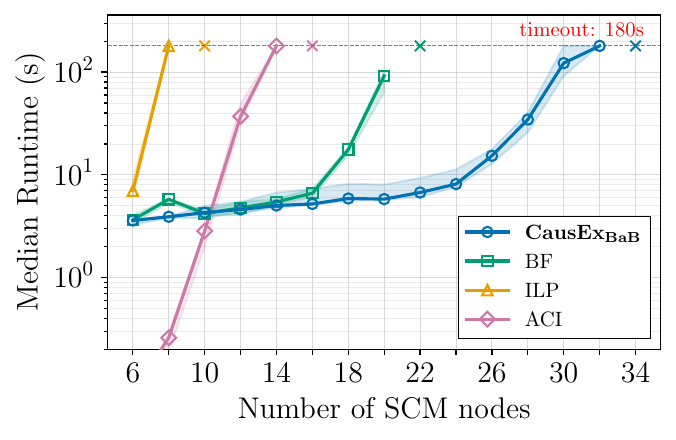}}
\caption{Runtime as a function of the number of SCM nodes. Lines show
median runtimes over nine runs per configuration, with shaded regions
indicating the interquartile range. A point is marked with $\times$ if
that configuration and all subsequent ones time out for that method.}
\label{fig:nodes}
\end{figure}

\paragraph{Number of SCM nodes (\figref{fig:nodes}).}
We vary the number of nodes from 6 to 34 in increments of 2, compute all
minimal causes, and use the \texttt{m4}
architecture\footnote{\tabref{tab:network_configs} (\appref{app:networksize}) lists the full range of tested architectures.} ($256\times128\times64\times32$). \causexbab completes all
runs up to $n=28$, with median runtime increasing from 3.6s to 34.5s.
At $n=30$, 33\% of runs time out; at $n=32$, 88.9\% time out; and all runs time out from $n=34$.
\brutef completes all runs up to $n=20$ but times out from $n=22$
onward. At $n=20$, its median runtime is 92.0s, vs.\ 5.8s
for \causexbab. \ilp times out on one run at $n=6$, on 89\% of runs at
$n=8$, and on all runs from $n=10$. \aci is fastest for small graphs
($n\leq10$, median below 3s), but its runtime rises sharply at $n=12$
and it times out on the majority of runs from $n=14$ and on all runs from
$n=16$. Wilcoxon signed-rank tests show that \causexbab is significantly
faster than \ilp and \aci at almost every tested $n$ ($p<0.05$), while
the difference from \brutef becomes significant from $n=18$ onward.
Across the 115 completed \causexbab runs, 68.7\% yield multiple minimal
causes, 13.0\% yield exactly one, and 18.3\% yield none.

\paragraph{Further Results.}
Additional experiments confirm the importance of causal structure and
show that \causexbab remains robust across network, graph, and search
settings.

\noindent
\textit{Independence baseline.} Treating input features as causally independent finds $59\%$ more causes that are $20\%$ larger on average, recovers only $33\%$ of the true minimal causes, and $74.3\%$ of what it finds is valid but not minimal under the \ac{scm} (\appref{app:independence}).

\noindent
\textit{Neural network size.}
Varying the NN architecture from 96 to 8160 hidden ReLU units while fixing the SCM at 10 nodes has little
effect on the runtimes of \causexbab, \brutef, and \aci, whereas \ilp
times out on the majority of runs (\appref{app:networksize}).

\noindent
\textit{Scalability.} Restricting the maximum cause size searched to $k_{\max}=8$, \causexbab completes almost all runs up to $n=38$ and still solves instances at $n=56$ (\appref{app:scalability}).

\noindent
\textit{Connected components.} Runtime is driven primarily by graph size, while the number of components has a smaller, non-monotone effect that is largely explained by variation in the number of minimal causes to enumerate (\appref{app:components}).

\noindent
\textit{Sparsity.} Runtime shows no systematic dependence on graph density, with median runtimes remaining within $26$--$47$s across $m{=}1,\ldots,8$ (\appref{app:sparsity}).

\noindent
\textit{Branch selection.} The default \texttt{width\_large} heuristic clearly outperforms the alternatives, resolving $51.2\%$ of undecided variables per pruned branch and avoiding all timeouts at the largest tested size (\appref{app:branchselection}).

\noindent
\textit{Maximum cause size.} Increasing $k_{\max}$ from $1$ to $8$ leaves \causexbab's median runtime essentially unchanged at $5.3$--$5.5$s, whereas the baselines become slower (\appref{app:maxk}).

\noindent
\textit{Cause sizes.} Minimal causes span a broad range of sizes; restricting the search to $k_{\max}{=}8$ would retain only $59.4\%$ of the causes found by unrestricted search (\appref{app:causesizes}).

\subsection{Case Study: SNAP Quality-Control Targeting}
As a case study, we build an \ac{scm} from the Certification of Eligible Households Regulation~\cite{CFR273} of the Supplemental Nutrition Assistance Program \cite{Leftin2026} and use it to compute causes for decisions made by an \ac{nn}, trained on USDA's SNAP Quality Control data \cite{USDAFNS2026}, that flags certified households for benefit-error reviews. \causexbab computes actual causes for every flagged household in a median runtime of $4.60$s, while \brutef takes a median runtime of $43.24$s and \ilp and \aci time out on $3/11$ and $10/11$ instances respectively. Treating the inputs as independent instead of respecting the \ac{scm} raises the median number of causes found from $23$ to $50$, and $14.9\%$ of them turn out to be spurious. See \appref{app:casestudy}.

\section{Related Work}
\label{sec:related}

 HP actual causality has been extensively studied in discrete, non-neural models~\cite{Leitner-Fischer2013,Coenen2022,Ibrahim2020}, but computing HP causes for \ac{nn} predictions under structured causal inputs remains underexplored.
Closest to our setting, \citet{Ibrahim2020} compute exact HP causes via an ILP encoding that scales poorly with neural network components. \citet{Reyd2025} search heuristically, without formal guarantees and possibly returning incomplete or non-minimal sets. \citet{Chockler2025} compute multiple actual causes for black-box image classifiers via repeated querying, but assume independent features. \citet{Aryan2026} use an incomplete continuous gradient-based search. We instead provide formal guarantees along three dimensions: (i) minimality of the computed causes, (ii) minimality and sufficiency of the used contingencies, and (iii) computation of all valid HP actual causes, over an SCM-structured
input space, while scaling to substantially larger instances than the baselines.

Feature-attribution methods such as LIME~\cite{Ribeiro2016} and SHAP~\cite{Lundberg2017},
  and causal variants that relax the independence assumption via a given causal
  graph~\cite{Heskes2020}, return importance scores rather than HP actual causes.
\emph{Minimal abductive explanations}~\cite{Bassan2023} are a dual notion, where a fixed value alone guarantees the prediction rather than a change alone altering it.

\section{Discussion}
\label{sec:discussion}

We keep the SCM separate from the NN because the two components serve
different roles. The SCM provides explicit intervention semantics and
domain knowledge, while the NN remains the fixed predictive model whose
output is being explained. Although the NN may learn statistical
dependencies that overlap with those encoded by the SCM, these
dependencies are not explicitly represented as causal mechanisms.
The two components could instead be combined in a neural SCM, but this
would require a model specifically learned as an SCM
rather than an arbitrary pretrained predictor.

\paragraph{Beyond Boolean SCMs.}

 For continuous SCMs, verification methods can still reason soundly about continuous input regions, but exact enumeration of all minimal HP causes no longer reduces to a finite search over assignments. One must specify the admissible intervention space and the relevant notion of minimality. Extending our guarantees to that setting is therefore an important but separate problem. Towards this, we experimented with tighter bounds via $\beta$-CROWN~\cite{Wang2021}. Each iteration's bound computation took substantially longer to run than IBP, and did not improve overall runtime in our Boolean setting. We conjecture it may help once inputs are continuous.

\section{Conclusion}
\label{sec:conc}

We formalized NN explanation under structured input dependencies as the computation of Halpern--Pearl actual causes over SCM-constrained inputs. 
We introduced \causexbab, a sound and complete algorithm that returns all minimal actual causes together with minimum-cardinality witness contingencies. Across SCM--NN benchmarks, \causexbab substantially outperforms brute-force, ILP-based, and heuristic search baselines
as graph size grows.
Future work includes an extension to continuous domains as well as bounding the maximum cause size 
automatically from the graph structure,
an evaluation of our approach on larger networks and additional real-world datasets, and integrating gradient-based candidate generation with branch-and-bound refinement.

\section*{Acknowledgements}
The work of Leue, Azeem and Strobel was partially funded through the DFG research grant LE 1342/4 \enquote{SCADNet - Structural Causal Analysis of Deep Neural Networks}. 

\bibliography{references}

\clearpage

\appendix

\section{Detailed Preliminaries}
\label{app:prelims}
\subsection{\acp{scm}}

We use \aclp{scm} to represent and reason about the causal relationships between variables in a system.

An \ac{scm} is a tuple $\mathcal{M} = (\vec{U}, \vec{V}, \mathcal{F})$~\cite{Pearl2009} with exogenous variables $\vec{U} = \{U_1, \dots, U_d\}$, endogenous variables $\vec{V} = \{X_1, \dots, X_d\}$, and structural equations $\mathcal{F} = \{f_1, \dots, f_d\}$ such that

\begin{equation}
	X_j := f_j(\mathrm{PA}_j, U_j), \quad j = 1, \dots, d,
\end{equation}

where $\mathrm{PA}_j \subseteq \vec{V} \setminus \{X_j\}$ denotes the set of endogenous parents (direct causes) of $X_j$.

This formulation defines how each variable is generated from its direct causes and some exogenous influence. It captures both the causal structure (via the dependencies among variables) and the functional relationships that determine how each variable is computed from its causes.

Given a specific assignment $\vec{u} = (u_1, \dots, u_d)$ to the exogenous variables, called a \emph{context}, the values of the endogenous variables $\vec{V}$ are determined by recursive application of the structural equations.

Each SCM corresponds to a \emph{directed acyclic graph} (DAG) where each edge represents a direct causal influence. The nodes are represented by the variables $X_j$ while the (directed) edges are represented by $X_i \rightarrow X_j \Leftrightarrow X_i \in \mathrm{PA}_j$.

\paragraph{Binary \acp{scm}.}
A binary \ac{scm} is a \ac{scm} in which all endogenous variables take values in the Boolean domain, i.e., $X_j \in \{0,1\}$ for $j = 1,\dots,d$. The exogenous variables $U_j$ need not be binary. Each structural equation $f_j$ is a Boolean function, typically expressed using the operators $\land$, $\lor$, and~$\lnot$.

\paragraph{Interventions.}
An intervention refers to setting a variable to a specific value, independent of its usual causes. Given an endogenous variable $X := f_X(\mathrm{PA}_X, U_X)$, the intervention $X \leftarrow x$ modifies the model by replacing this equation with the assignment $X := x$. This enables us to quantify the causal effect of a variable by analyzing how the system behaves when it is manipulated externally, regardless of its original causes.

\paragraph{Counterfactuals.}
Interventions modify the \ac{scm} so that we can analyze how changes to one variable affect other variables at the population level. In contrast, counterfactuals explain "What would have happened if~\dots?", focusing on individual-level scenarios. They are computed by first determining the values of the exogenous variables from the observed effect and then applying an intervention to the model.

\subsection{Halpern and Pearl's Actual Causality}
SCMs allow us to compute the effects of interventions and simulate alternative scenarios but do not specify what counts as a cause in a particular instance, especially when multiple variables influence the effect. 
 Halpern and Pearl's definition of actual causality extends this framework by providing criteria to identify a minimal, explanatory cause of a specific observed event. This enables reasoning about responsibility and supports the generation of fine-grained causal explanations.

We adopt the modified definition of actual causality~\cite{Halpern2015}, which avoids known problems with preemption and spurious causes at lower complexity than earlier versions of the definition.

\paragraph{Definition.}
We denote the set of endogenous variables by $\vec{V} = \{X_1, \dots, X_d\}$. In the following, we use Halpern and Pearl's notation, where subsets of $\vec{V}$ are used to define candidate causes and counterfactual contingencies.

A \emph{context} $\vec{u}$ is a specific assignment to all exogenous variables in a SCM. Fixing a context determines the values of all endogenous variables via the structural equations, and therefore defines the actual world in which causal relationships are evaluated.

Let $(\mathcal{M}, \vec{u})$ be a causal model with context $\vec{u}$. A set of variables $\vec{X}$ taking values $\vec{x}$ is an \emph{actual cause} of an outcome $\varphi$ in $(\mathcal{M}, \vec{u})$ if the following conditions are satisfied:

\begin{enumerate}
    \item[]\textbf{AC1}
    $(\mathcal{M}, \vec{u}) \models \vec{X} = \vec{x}$ and $(\mathcal{M}, \vec{u}) \models \varphi$.  
    That is, $\vec{X} = \vec{x}$ is only considered a cause of $\varphi$ if both the assignment and the outcome actually occur in the context $\vec{u}$.

    \item[]\textbf{AC2}
    There exists a set $\vec{W} \subseteq \vec{V} \setminus \vec{X}$ of endogenous variables and alternative values $\vec{x}' \neq \vec{x}$ for $\vec{X}$, such that if $\vec{w}$ denotes the actual values of the variables in $\vec{W}$ in context $(\mathcal{M}, \vec{u})$, then
    \begin{equation*}
    	(\mathcal{M}, \vec{u}) \models [\vec{X} \leftarrow \vec{x}',\, \vec{W} \leftarrow \vec{w}]\, \lnot \varphi.
    \end{equation*}
    This means that by holding a set of variables fixed at their actual values, changing $\vec{X}$ alone is enough to prevent $\varphi$ from occurring.

The triple $(\vec{W}, \vec{w}, \vec{x}')$ serves as a \emph{witness} showing that $\vec{X} \leftarrow \vec{x}'$ satisfies condition AC2 and thus qualifies as a cause of $\varphi$.

    \item[]\textbf{AC3} 
    $\vec{X}$ is minimal. No subset of $\vec{X}$ satisfies both AC1 and AC2.
    
\end{enumerate}
The modified definition restricts contingencies to the variables' actual values in the given context $\vec{u}$, rather than allowing arbitrary counterfactual settings. This prevents a variable from being designated a cause based only on a constructed contingency that never actually occurred, and instead focuses on whether the candidate cause made a difference under the conditions that actually held.

\section{Running Example: Additional Context}
\label{app:running_example_orig}

We show \causexbab on another context of our running example (\secref{subsec:runningexample}, \figref{fig:loan_scm}): $\vec u = (0,1,1,1,1)$. This induces the factual assignment $\vec x = (0,0,1,1,1)$\footnote{Computed in topological order: $X_1=U_1=0$, $X_3=X_1\oplus U_3=1$, $X_5=U_5=1$, $X_2=(X_3\wedge\lnot X_1)\oplus U_2=0$, $X_4=(X_2\wedge\lnot X_5)\oplus U_4=1$.}: the applicant has no high income, no negative cashflow, high expenses, unsecured exposure, and a guarantor. The predictor outputs $y = N(\vec x) \ge \tau$, so the bank classifies the applicant as \emph{high risk}. As before, our goal is to find minimal sets of features $\vec X \subseteq \{X_1, \ldots, X_5\}$ that are actual causes of this classification, that is, interventions that, under some contingency, would push the prediction below $\tau$ and flip the applicant to \emph{low risk}.

We illustrate this procedure on this context (\figref{fig:trees-orig}). Initially, $X_1$ is fixed to \textsc{cause} and every other variable is left \textsc{undecided}. The initial bounds, $[0.07, 0.96]$, are inconclusive, so the search branches on $X_3$. Both resulting children ($X_3$ is \textsc{contingency} and $X_3$ is \textsc{natural}) are inconclusive. Branching on $X_2$ in the branch where $X_3$ is \textsc{natural} shows that both children falsify $\varphi$, confirming $\{X_1\}$ as a cause. We select the one with the smaller contingency.
In the branch where $X_3$ is \textsc{contingency}, both children of $X_2$ satisfy $\varphi$ and are pruned. For $\{X_2\}$, the bounds already resolve to a single point, $0.97$, since fixing $X_2$ leaves nothing else undetermined, and this point satisfies $\varphi$, so $\{X_2\}$ is pruned without any branching. $\{X_3\}$ is inconclusive and branches once on $X_2$, but both children satisfy $\varphi$ and are pruned. $\{X_4\}$ and $\{X_5\}$ resolve to a single point like $\{X_2\}$ and are pruned the same way. This leaves $\{X_1\}$ as the only confirmed cause of size $k=1$.

\renewcommand{\treescale}{0.6}
\begin{figure}[htb]
\centering
\scalebox{\treescale}{
\begin{tikzpicture}[
  level distance=12mm,
  level 1/.style={sibling distance=16mm},
  level 2/.style={sibling distance=14mm},
  level 3/.style={sibling distance=18mm},
  level 4/.style={sibling distance=8mm},
]
\node[root] {}
  child { node[base] {$\{X_1\}$}
    child { node[branch] {$X_3$}
      child { node[branch] {$X_2$} edge from parent[draw=treeNaturalEdge]
        child { node[confirmed] {C} edge from parent[draw=treeNaturalEdge] }
        child { node[confirmed] {C} edge from parent[draw=treeContingencyEdge] }
      }
      child { node[branch] {$X_2$} edge from parent[draw=treeContingencyEdge]
        child { node[pruned] {P} edge from parent[draw=treeNaturalEdge] }
        child { node[pruned] {P} edge from parent[draw=treeContingencyEdge] }
      }
    }
  }
  child { node[base] {$\{X_2\}$}
    child { node[pruned] {P} }
  }
  child { node[base] {$\{X_3\}$}
    child { node[branch] {$X_2$}
      child { node[pruned] {P} edge from parent[draw=treeNaturalEdge] }
      child { node[pruned] {P} edge from parent[draw=treeContingencyEdge] }
    }
  }
  child { node[base] {$\{X_4\}$}
    child { node[pruned] {P} }
  }
  child { node[base] {$\{X_5\}$}
    child { node[pruned] {P} }
  };
\end{tikzpicture}
}

\vspace{4mm}

\scalebox{\treescale}{
\begin{tikzpicture}[
  level distance=15mm,
  level 1/.style={sibling distance=20mm},
  level 2/.style={sibling distance=16mm},
  level 3/.style={sibling distance=10mm},
  level 4/.style={sibling distance=8mm},
]
\node[root] {}
  child { node[skip] {$\{X_1,\ldots\}$} }
  child { node[base] {$\{X_2,X_3\}$}
    child { node[pruned] {P} }
  }
  child { node[base] {$\{X_2,X_4\}$}
    child { node[pruned] {P} }
  }
  child { node[base] {$\{X_2,X_5\}$}
    child { node[pruned] {P} }
  }
  child { node[base] {$\{X_3,X_4\}$}
    child { node[branch] {$X_2$}
      child { node[pruned] {P} edge from parent[draw=treeNaturalEdge] }
      child { node[confirmed] {C} edge from parent[draw=treeContingencyEdge] }
    }
  }
  child { node[base] {$\{X_3,X_5\}$}
    child { node[branch] {$X_2$}
      child { node[branch] {$X_4$} edge from parent[draw=treeNaturalEdge]
        child { node[pruned] {P} edge from parent[draw=treeNaturalEdge] }
        child { node[pruned] {P} edge from parent[draw=treeContingencyEdge] }
      }
      child { node[pruned] {P} edge from parent[draw=treeContingencyEdge] }
    }
  }
  child { node[base] {$\{X_4,X_5\}$}
    child { node[pruned] {P} }
  };
\end{tikzpicture}
}
\caption{\causexbab search trees for $k=1$ (top) and $k=2$ (bottom) on this context. At each branch, the left child resolves the branching variable to \textcolor{treeNaturalEdge}{\textsc{natural}} (blue) and the right child to \textcolor{treeContingencyEdge}{\textsc{contingency}} (orange). \textcolor{treeConfirmedEdge}{C = confirmed}, \textcolor{treePrunedEdge}{P = pruned}. In the $k=2$ tree, candidates containing $X_1$ are skipped as supersets of the already-confirmed $\{X_1\}$.}
\label{fig:trees-orig}
\end{figure}
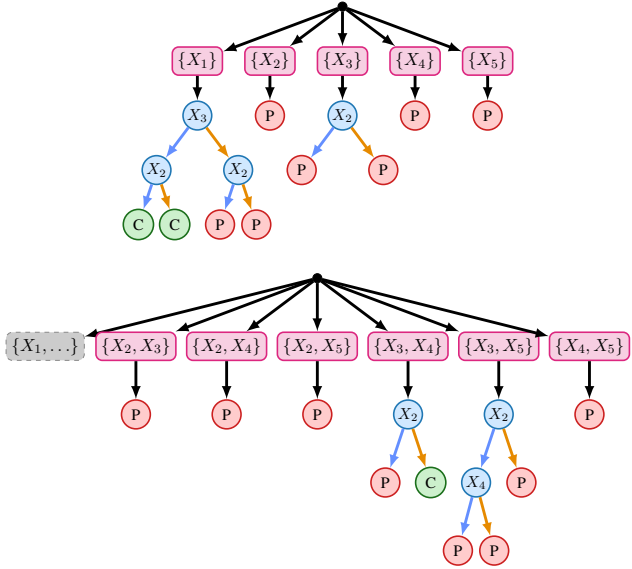

Enumerating potential causes for $k=2$, candidates containing $X_1$ are skipped to satisfy AC3 (minimality). $\{X_2,X_3\}$, $\{X_2,X_4\}$, and $\{X_2,X_5\}$ are each conclusive and pruned immediately, without branching. $\{X_3,X_4\}$ is inconclusive and branches once on $X_2$: the \textsc{natural} child satisfies $\varphi$ and is pruned, while the \textsc{contingency} child falsifies $\varphi$, confirming $\{X_3,X_4\}$ as a second cause with witness contingency $\{X_2\}$. $\{X_3,X_5\}$ is inconclusive and also branches on $X_2$, where its \textsc{natural} child needs a further branch on $X_4$ to be conclusive. All resulting children satisfy $\varphi$, so $\{X_3,X_5\}$ is pruned entirely. $\{X_4,X_5\}$ is conclusive and pruned immediately. This leaves two minimal causes for this context: $\{X_1\}$ (\emph{high income}) and $\{X_3,X_4\}$ (\emph{high expenses}, \emph{unsecured exposure}) with witness contingency $\{X_2\}$ (\emph{negative cashflow}).

If the five features are instead treated as causally independent (no edges between $X_1, \ldots, X_5$), the same modified Halpern--Pearl definition yields three minimal causes: $\{X_1,X_3\}$, $\{X_1,X_4\}$, and $\{X_3,X_4\}$, each without a contingency, since \textsc{natural} and \textsc{contingency} coincide once every variable is a root. The minimal sufficient sets for the same instance, a distinct dual notion (a set whose fixed value alone guarantees the classification regardless of the rest), happen to coincide exactly with these three sets.
Modeling the causal structure changes which causes are found. $\{X_1\}$ alone qualifies as a cause only because its effect propagates through $X_3,X_2,X_4$, whereas under independence no single feature suffices. Methods that treat inputs as independent will therefore often only find larger causes, missing that a smaller one can already be responsible once its downstream propagation is accounted for.

\section{Additional Details on the Approach}
\label{app:approach}

\begin{remark}\label{rem:responsibility}
Halpern's \emph{degree of responsibility} of $C$ as a cause of $\varphi$ is defined as $1/(N+1)$, where $N$ is the size of the smallest witnessing contingency~\cite{Halpern2016}. Since Step~2 searches for exactly this minimal $N$, $\deg(C)$ is obtained directly as a byproduct of the search. The smaller the contingency, the higher the responsibility.
\end{remark}

\subsection{Proof of Theorem~\ref{thm:correctness}}
\label{app:proofs}

\thmcorrectness*
\begin{proof}
\textit{Termination.} The algorithm considers finitely many candidate cause sets. For a fixed candidate $C$, every recursive branch resolves one previously \textsc{undecided} variable by assigning it either the \textsc{natural} or \textsc{contingency} role, so every root-to-leaf path has length at most $|\vec V|-|C|$, and the search tree is finite. Therefore \algref{alg:causexbab} terminates.

\smallskip

\noindent\textit{Correctness.}
\textit{AC1.} Every candidate $C$ is evaluated in the actual context $\vec u$, where $C$ trivially takes its actual value, and $\varphi$ holds since $N(\vec x) \in S$ by assumption. Hence every candidate satisfies AC1.

\smallskip
\noindent\textit{AC2.} Let $C$ be returned via some confirmed region $\mathcal B$, i.e., $[\ell,h] \cap S = \emptyset$. Given $\mathrm{bound}_{\mathrm{SCM}}$ and $\mathrm{bound}_{\mathrm{NN}}$ are sound, every completion represented by $\mathcal B$ falsifies $\varphi$. Let $W_{\mathcal B} = \{X \in \vec V \setminus C : \mathcal B(X) = \textsc{contingency}\}$ be the contingency variables already committed in $\mathcal B$. Resolving every remaining \textsc{undecided} variable to \textsc{natural} yields one such completion, namely the intervention $[C \leftarrow \vec x'_C, W_{\mathcal B} \leftarrow \vec w_{\mathcal B}]$, where $\vec w_{\mathcal B}$ are the actual values of $W_{\mathcal B}$. Since this completion is represented by $\mathcal B$, it falsifies $\varphi$, so $W_{\mathcal B}$ is a valid witness contingency and $C$ satisfies AC2.

\smallskip
\noindent\textit{AC3 and completeness.} We prove jointly, by induction on $k=|C|$, that after all candidates of size at most $k$ have been processed: (i) every returned candidate of size at most $k$ satisfies AC1--AC3, and (ii) every actual cause of size at most $k$ has been returned.

\emph{Base case ($k=1$).} Every returned candidate satisfies AC1 and AC2 by the argument above. A singleton has no proper nonempty subset, so it trivially satisfies AC3. Conversely, let $C$ be an actual cause with $|C|=1$; since variables are binary, its unique witness value is the complement $\vec x'_C$. It cannot be skipped, since no smaller candidate has yet been processed, and the per-candidate search described under AC2 is exhaustive over contingencies, so its witness is found and $C$ is returned.

\emph{Inductive step.} Assume (i) and (ii) hold for all sizes smaller than $k$; consider size $k$.

For (i), let $C$ be a returned candidate with $|C|=k$; it satisfies AC1--AC2 as above. Suppose toward contradiction that some $D \subsetneq C$ satisfies AC1--AC2. Since $C$ is finite, $D$ contains an inclusion-minimal such subset $D^*$, which then satisfies AC3 too, i.e., $D^*$ is an actual cause with $|D^*| < k$. By the induction hypothesis, $D^*$ was already returned, so $C$ would have been skipped as its superset -- a contradiction. Hence $C$ satisfies AC3.

For (ii), let $C$ be an actual cause with $|C|=k$. By the induction hypothesis, every candidate already returned is a genuine actual cause of smaller size; since $C$ is minimal, none of these is a subset of $C$, so $C$ is not skipped. We next argue that $C$'s witness may be taken to be the full complement $\vec x'_C$: if some AC2 witness for $C$ left a variable $X_j \in C$ at its actual value, then $C \setminus \{X_j\}$, together with $X_j$ added to the contingency, would witness the same intervention and hence also satisfy AC1--AC2 -- contradicting the minimality of $C$. So $C$'s witness contingency $W$ pairs with the complement $\vec x'_C$, and the corresponding role assignment is represented by the initial region for $C$. This region cannot be pruned by output bounds, since its completion falsifies $\varphi$, nor by contingency cardinality until a witness of no greater size has already been secured for $C$. The search therefore finds a witness for $C$, and $C$ is returned.

This establishes (i) and (ii) for size $k$, completing the induction. Hence \algref{alg:causexbab} returns exactly the set of all actual causes.

\smallskip
\noindent\textit{Minimum-cardinality contingency.} When a region $\mathcal B$ is confirmed, resolving every remaining \textsc{undecided} variable to \textsc{natural} yields contingency $W_{\mathcal B}$. Every other completion represented by $\mathcal B$ has contingency $W_{\mathcal B} \cup W'$ for some $W'$, and hence cardinality at least $|W_{\mathcal B}|$, so $W_{\mathcal B}$ is the smallest contingency represented by $\mathcal B$. Since the search continues after a first witness is found, discarding only regions whose committed contingency already has cardinality at least $|W^*|$ for the current incumbent $W^*$ -- and no such region can improve on $W^*$, by the previous sentence -- the contingency retained when the search for $C$ terminates has minimum cardinality among all witnesses for $C$.
\end{proof}

\begin{remark}
\label{rem:numericalprecision}
The theoretical guarantees above assume exact arithmetic. In the implementation, every fully resolved witness is re-evaluated using high-precision arithmetic before being returned. To avoid floating-point error causing an incorrect prune or confirmation, decisions within a numerical tolerance of the acceptance-set boundary are treated as inconclusive and refined further. This adds negligible overhead, since the recheck only fires once per returned witness, and rules out large numerical errors, though errors within the tolerance margin itself are not eliminated.
\end{remark}

\begin{remark}
\label{rem:kmax}
The algorithm can optionally restrict the outer search to candidate sizes $k=1,\ldots,k_{\max}$ instead of $k=1,\ldots,|\vec V|$, to terminate early. In that case, the guarantee above only holds up to $k_{\max}$. A minimal cause larger than $k_{\max}$ is then never enumerated as a candidate. Bounding $k_{\max}$ is often desirable regardless of runtime, since smaller cause sets are more likely to be interpretable and actionable by a human, whereas a very large cause set may be less desirable as an explanation even when valid.
\end{remark}

\begin{remark}
\label{rem:satisfactionset}
Halpern and Pearl's definition does not specify what counts as an outcome $\varphi$ for continuous-valued models such as our NN predictor. The acceptance set $S$ is how we close this gap. The soundness and completeness above hold for any choice of $S$: changing $S$ changes which outcome $\varphi$ \algref{alg:causexbab} decides and the concrete prune/confirm conditions it uses, but not the algorithm itself or its guarantees. A returned pair $(C, \vec w^*)$ is thus guaranteed to be an actual cause \emph{of this specified outcome}.
\end{remark}

\section{Complexity Analysis}
\label{app:complexity}

For a candidate cause $C$ with $|C|=k$, full enumeration checks every possible contingency by assigning each of the $|\vec V|-k$ remaining variables to \textsc{natural} or \textsc{contingency}, giving $2^{|\vec V|-k}$ assignments per candidate. Summed over all candidate sizes $k=1,\dots,|\vec V|$, this yields a worst-case search space of
\begin{equation*}
	{\textstyle \sum_{k=1}^{|\vec V|} \binom{|\vec V|}{k} \cdot 2^{|\vec V|-k} \;\leq\; 3^{|\vec V|},}
\end{equation*}
i.e., in the worst case every intervenable variable is either a \textsc{cause}, a \textsc{contingency}, or \textsc{natural}.

This worst-case bound is consistent with the underlying decision problem's known intractability. Deciding whether a set of variables is an actual cause is NP-complete for binary models and $\Sigma_2^P$-complete for general models under Halpern and Pearl's original definition~\cite{Eiter2001}, and $D_2^P$-complete under the modified definition we adopt~\cite{Aleksandrowicz2017}. No polynomial-time algorithm is therefore expected to exist for our problem under standard complexity-theoretic assumptions.

\causexbab searches the same space, but not by directly enumerating full assignments. For a candidate cause $C$ with $|C|=k$, the search resolves the $n=|\vec V|-k$ non-cause variables one at a time: each region branches into two children, one resolving the next \textsc{undecided} variable to \textsc{natural} and one to \textsc{contingency}. In the worst case, where a candidate stays inconclusive until every variable is resolved, this produces a full binary search tree of depth $n$ containing $\sum_{i=0}^{n}2^i=2^{n+1}-1$ regions in total: $2^n$ fully-resolved leaves, matching the $2^{|\vec V|-k}$ assignments full enumeration also checks for this candidate, plus $2^n-1$ additional internal regions in which at least one non-cause variable is still \textsc{undecided}, a state full enumeration never represents, since it only ever constructs fully-resolved assignments. Summed over all candidate sizes, the overall number of regions \causexbab visits is therefore at most
\begin{equation*}
	{\textstyle \sum_{k=1}^{|\vec V|}\binom{|\vec V|}{k}\left(2^{|\vec V|-k+1}-1\right) \;<\; 2\cdot3^{|\vec V|},}
\end{equation*}
the same order, $O(3^{|\vec V|})$, as full enumeration's space, but with roughly twice as many distinct states to visit: every candidate's $2^{|\vec V|-k}-1$ internal \textsc{undecided} regions on top of the $2^{|\vec V|-k}$ leaves it shares with full enumeration.

Beyond visiting more states, each one, leaf or internal, also costs more: it requires a call to $\mathrm{bound}_{\mathrm{SCM}}$ and $\mathrm{bound}_{\mathrm{NN}}$, whereas full enumeration evaluates $N$ directly, once per leaf. Since a bound computation costs more than one such forward pass, our worst case can be more expensive per candidate than full enumeration in absolute terms, not just in state count. A real speedup thus requires most branches to be pruned or confirmed well before full depth, avoiding both the extra states and their higher per-state cost.

\section{Additional Experimental Results}
\label{app:additional_experiments}

Unless stated otherwise, each configuration is evaluated on 3 independently generated graphs and 3 sampled contexts per graph, yielding 9 runs per configuration. Lines in the runtime plots below show medians over these runs, with shaded regions indicating the interquartile range. A point is marked with $\times$ if that configuration and all subsequent ones timed out for that method. Graph, context, and teacher-network randomness are all derived deterministically from a single per-instance base seed.

\subsection{Number of SCM Nodes}
Full per-configuration results underlying \figref{fig:nodes} (\secref{sec:exp}) are given here. \tabref{tab:nodes} gives the node-scaling results.

\begin{table*}[t]
\centering
\caption{Median runtime (s) vs.\ number of SCM nodes $n$, computing all minimal causes ($k_{\max}=|V|$), \texttt{m4} architecture (Fig.~\ref{fig:nodes}).}
\label{tab:nodes}
\setbox0=\hbox{\begin{tabular}{lrrrrrrrrrrrrrrr}
\toprule
$n$ & 6 & 8 & 10 & 12 & 14 & 16 & 18 & 20 & 22 & 24 & 26 & 28 & 30 & 32 & 34 \\
\midrule
\causexbab & 3.6 & 3.9 & 4.3 & 4.6 & 5.0 & 5.2 & 5.9 & 5.8 & 6.7 & 8.1 & 15.3 & 34.5 & 122.3$^\dagger$ & 180.0$^\dagger$ & 180.0$^\ddagger$ \\
\brutef & 3.6 & 5.8 & 4.2 & 4.8 & 5.4 & 6.6 & 17.6 & 92.0 & 180.0$^\ddagger$ & 180.0$^\ddagger$ & 180.0$^\ddagger$ & 180.0$^\ddagger$ & 180.0$^\ddagger$ & 180.0$^\ddagger$ & 180.0$^\ddagger$ \\
\ilp & 7.0$^\dagger$ & 180.0$^\dagger$ & 180.0$^\ddagger$ & 180.0$^\ddagger$ & 180.0$^\ddagger$ & 180.0$^\ddagger$ & 180.0$^\ddagger$ & 180.0$^\ddagger$ & 180.0$^\ddagger$ & 180.0$^\ddagger$ & 180.0$^\ddagger$ & 180.0$^\ddagger$ & 180.0$^\ddagger$ & 180.0$^\ddagger$ & 180.0$^\ddagger$ \\
\aci & 0.1 & 0.3 & 2.8 & 37.1 & 180.0$^\dagger$ & 180.0$^\ddagger$ & 180.0$^\ddagger$ & 180.0$^\ddagger$ & 180.0$^\ddagger$ & 180.0$^\ddagger$ & 180.0$^\ddagger$ & 180.0$^\ddagger$ & 180.0$^\ddagger$ & 180.0$^\ddagger$ & 180.0$^\ddagger$ \\
\bottomrule
\end{tabular}}\ifdim\wd0>\textwidth
\resizebox{\textwidth}{!}{\box0}
\else
\begin{tabular*}{\textwidth}{@{\extracolsep{\fill}}lrrrrrrrrrrrrrrr@{}}
\toprule
$n$ & 6 & 8 & 10 & 12 & 14 & 16 & 18 & 20 & 22 & 24 & 26 & 28 & 30 & 32 & 34 \\
\midrule
\causexbab & 3.6 & 3.9 & 4.3 & 4.6 & 5.0 & 5.2 & 5.9 & 5.8 & 6.7 & 8.1 & 15.3 & 34.5 & 122.3$^\dagger$ & 180.0$^\dagger$ & 180.0$^\ddagger$ \\
\brutef & 3.6 & 5.8 & 4.2 & 4.8 & 5.4 & 6.6 & 17.6 & 92.0 & 180.0$^\ddagger$ & 180.0$^\ddagger$ & 180.0$^\ddagger$ & 180.0$^\ddagger$ & 180.0$^\ddagger$ & 180.0$^\ddagger$ & 180.0$^\ddagger$ \\
\ilp & 7.0$^\dagger$ & 180.0$^\dagger$ & 180.0$^\ddagger$ & 180.0$^\ddagger$ & 180.0$^\ddagger$ & 180.0$^\ddagger$ & 180.0$^\ddagger$ & 180.0$^\ddagger$ & 180.0$^\ddagger$ & 180.0$^\ddagger$ & 180.0$^\ddagger$ & 180.0$^\ddagger$ & 180.0$^\ddagger$ & 180.0$^\ddagger$ & 180.0$^\ddagger$ \\
\aci & 0.1 & 0.3 & 2.8 & 37.1 & 180.0$^\dagger$ & 180.0$^\ddagger$ & 180.0$^\ddagger$ & 180.0$^\ddagger$ & 180.0$^\ddagger$ & 180.0$^\ddagger$ & 180.0$^\ddagger$ & 180.0$^\ddagger$ & 180.0$^\ddagger$ & 180.0$^\ddagger$ & 180.0$^\ddagger$ \\
\bottomrule
\end{tabular*}
\fi
\vspace{2pt}
{\footnotesize $^\dagger$median over a batch that includes at least one timed-out run (capped at 180s) $^\ddagger$every run in the batch timed out\par}
\end{table*}

\subsection{Neural Network Size}
\label{app:networksize}
\figref{fig:network} shows the network-size runtime plot, \tabref{tab:network_configs} the architecture configurations, and \tabref{tab:network} the full per-architecture results.

\begin{table}[h]
\centering
\caption{Network architecture configurations used in the network scaling
         experiments. Total units refers to the sum of all hidden layer widths.}
\label{tab:network_configs}
\begin{tabular}{ll}
\toprule
\textbf{Label} & \textbf{Total units} \\
\midrule
\texttt{s2}  & 96    \\
\texttt{s3}  & 224   \\
\texttt{m4}  & 480   \\
\texttt{m5}  & 992   \\
\texttt{l6}  & 2016  \\
\texttt{l7}  & 4064  \\
\texttt{xl8} & 8160  \\
\bottomrule
\end{tabular}
\end{table}

\begin{table}[h]
\centering
\caption{Median runtime (s) vs.\ neural network architecture, $n=10$ nodes, computing all minimal causes ($k_{\max}=|V|$) (Fig.~\ref{fig:network}).}
\label{tab:network}
\setbox0=\hbox{\begin{tabular}{lrrrrrrr}
\toprule
Arch. & \texttt{s2} & \texttt{s3} & \texttt{m4} & \texttt{m5} & \texttt{l6} & \texttt{l7} & \texttt{xl8} \\
\midrule
\causexbab & 3.9 & 4.0 & 3.8 & 3.9 & 4.1 & 4.0 & 4.4 \\
\brutef & 4.1 & 3.7 & 3.7 & 3.8 & 4.6 & 4.2 & 4.1 \\
\ilp & 12.4 & 180.0$^\dagger$ & 180.0$^\ddagger$ & 180.0$^\ddagger$ & 180.0$^\ddagger$ & 180.0$^\ddagger$ & 180.0$^\ddagger$ \\
\aci & 2.6 & 2.3 & 1.9 & 2.6 & 2.6 & 2.9 & 2.6 \\
\bottomrule
\end{tabular}}\ifdim\wd0>\columnwidth
\resizebox{\columnwidth}{!}{\box0}
\else
\begin{tabular*}{\columnwidth}{@{\extracolsep{\fill}}lrrrrrrr@{}}
\toprule
Arch. & \texttt{s2} & \texttt{s3} & \texttt{m4} & \texttt{m5} & \texttt{l6} & \texttt{l7} & \texttt{xl8} \\
\midrule
\causexbab & 3.9 & 4.0 & 3.8 & 3.9 & 4.1 & 4.0 & 4.4 \\
\brutef & 4.1 & 3.7 & 3.7 & 3.8 & 4.6 & 4.2 & 4.1 \\
\ilp & 12.4 & 180.0$^\dagger$ & 180.0$^\ddagger$ & 180.0$^\ddagger$ & 180.0$^\ddagger$ & 180.0$^\ddagger$ & 180.0$^\ddagger$ \\
\aci & 2.6 & 2.3 & 1.9 & 2.6 & 2.6 & 2.9 & 2.6 \\
\bottomrule
\end{tabular*}
\fi
\vspace{2pt}
{\footnotesize $^\dagger$median over a batch that includes at least one timed-out run (capped at 180s) $^\ddagger$every run in the batch timed out\par}
\end{table}

\begin{figure}[h]
\centering
\resizebox{1.0\linewidth}{!}{\includegraphics{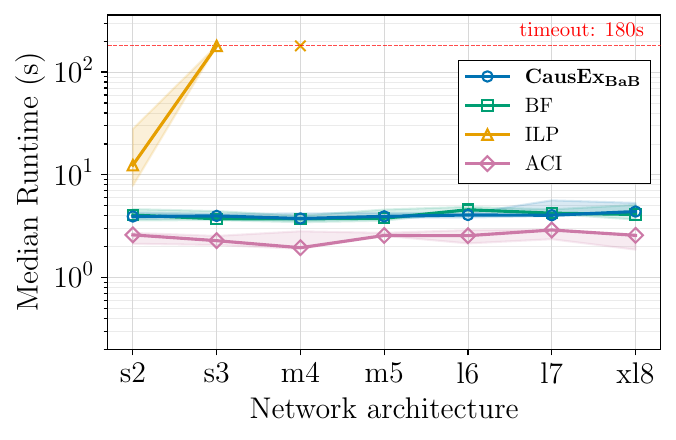}}
\caption{Runtime vs.\ neural network architecture, for all methods.}
\label{fig:network}
\end{figure}

\subsection{Statistical Significance}
\label{app:significance}

For the node-scaling and network-scaling experiments, we test whether \causexbab's runtime differs significantly from each baseline's using a Wilcoxon signed-rank test, matching runs by graph and context (i.e.\ comparing \causexbab and a baseline on the exact same graph instance and sampled context, not across different ones). \tabref{tab:significance_nodes} reports $p$-values vs.\ number of SCM nodes and \tabref{tab:significance_network} vs.\ network architecture. In the node-scaling experiment, the difference against \ilp and \aci is significant ($p<0.05$) at nearly every tested $n$, while the difference against \brutef only becomes significant from $n=18$ nodes onward, matching its close runtimes at smaller graph sizes. In the network-scaling experiment, \causexbab remains significantly faster than \ilp at every tested architecture, but \aci is significantly faster than \causexbab throughout, and the difference against \brutef is not significant at any architecture, consistent with \brutef's close (if slightly higher) runtimes there.

\begin{table*}[t]
\centering
\caption{Paired Wilcoxon signed-rank test ($p$-value) of \causexbab's runtime against each baseline, matched by graph and context, vs.\ number of SCM nodes $n$ (cf.\ Table~\ref{tab:nodes}).}
\label{tab:significance_nodes}
\setbox0=\hbox{\begin{tabular}{lrrrrrrrrrrrrrrr}
\toprule
$n$ & 6 & 8 & 10 & 12 & 14 & 16 & 18 & 20 & 22 & 24 & 26 & 28 & 30 & 32 & 34 \\
\midrule
\brutef & 0.203 & 0.004 & 0.652 & 0.820 & 0.426 & 0.074 & 0.004 & 0.004 & 0.004 & 0.004 & 0.004 & 0.004 & 0.031$^\dagger$ & -- & -- \\
\ilp & 0.004 & 0.004 & 0.004 & 0.004 & 0.004 & 0.004 & 0.004 & 0.004 & 0.004 & 0.004 & 0.004 & 0.004 & 0.031$^\dagger$ & -- & -- \\
\aci & 0.004 & 0.004 & 0.004 & 0.004 & 0.004 & 0.004 & 0.004 & 0.004 & 0.004 & 0.004 & 0.004 & 0.004 & 0.031$^\dagger$ & -- & -- \\
\bottomrule
\end{tabular}}\ifdim\wd0>\textwidth
\resizebox{\textwidth}{!}{\box0}
\else
\begin{tabular*}{\textwidth}{@{\extracolsep{\fill}}lrrrrrrrrrrrrrrr@{}}
\toprule
$n$ & 6 & 8 & 10 & 12 & 14 & 16 & 18 & 20 & 22 & 24 & 26 & 28 & 30 & 32 & 34 \\
\midrule
\brutef & 0.203 & 0.004 & 0.652 & 0.820 & 0.426 & 0.074 & 0.004 & 0.004 & 0.004 & 0.004 & 0.004 & 0.004 & 0.031$^\dagger$ & -- & -- \\
\ilp & 0.004 & 0.004 & 0.004 & 0.004 & 0.004 & 0.004 & 0.004 & 0.004 & 0.004 & 0.004 & 0.004 & 0.004 & 0.031$^\dagger$ & -- & -- \\
\aci & 0.004 & 0.004 & 0.004 & 0.004 & 0.004 & 0.004 & 0.004 & 0.004 & 0.004 & 0.004 & 0.004 & 0.004 & 0.031$^\dagger$ & -- & -- \\
\bottomrule
\end{tabular*}
\fi
\vspace{2pt}
{\footnotesize ``--'' marks configurations where the test is undefined (every matched pair's runtime difference is exactly zero, i.e.\ all methods time out identically), or based on fewer than 6 pairs, too few for any two-sided $p$-value to reach $0.05$. $^\dagger$marks $p$-values based on 6--8 pairs, since pairs where both methods time out are dropped as uninformative (identical capped runtime). With 9 paired runs per configuration, the smallest attainable two-sided $p$-value is $0.004$.\par}
\end{table*}

\begin{table}[h]
\centering
\caption{Paired Wilcoxon signed-rank test ($p$-value) of \causexbab's runtime against each baseline, matched by graph and context, vs.\ neural network architecture (cf.\ Table~\ref{tab:network}).}
\label{tab:significance_network}
\setbox0=\hbox{\begin{tabular}{lrrrrrrr}
\toprule
Arch. & \texttt{s2} & \texttt{s3} & \texttt{m4} & \texttt{m5} & \texttt{l6} & \texttt{l7} & \texttt{xl8} \\
\midrule
\brutef & 0.820 & 0.910 & 0.426 & 0.734 & 0.074 & 1.000 & 0.359 \\
\ilp & 0.008 & 0.004 & 0.004 & 0.004 & 0.004 & 0.004 & 0.004 \\
\aci & 0.004 & 0.004 & 0.004 & 0.004 & 0.004 & 0.004 & 0.004 \\
\bottomrule
\end{tabular}}\ifdim\wd0>\columnwidth
\resizebox{\columnwidth}{!}{\box0}
\else
\begin{tabular*}{\columnwidth}{@{\extracolsep{\fill}}lrrrrrrr@{}}
\toprule
Arch. & \texttt{s2} & \texttt{s3} & \texttt{m4} & \texttt{m5} & \texttt{l6} & \texttt{l7} & \texttt{xl8} \\
\midrule
\brutef & 0.820 & 0.910 & 0.426 & 0.734 & 0.074 & 1.000 & 0.359 \\
\ilp & 0.008 & 0.004 & 0.004 & 0.004 & 0.004 & 0.004 & 0.004 \\
\aci & 0.004 & 0.004 & 0.004 & 0.004 & 0.004 & 0.004 & 0.004 \\
\bottomrule
\end{tabular*}
\fi
\vspace{2pt}
{\footnotesize With 9 paired runs per configuration, the smallest attainable two-sided $p$-value is $0.004$.\par}
\end{table}

\subsection{ER Graphs: Number of SCM Nodes}
\label{app:er_results}

The main paper's benchmarks (\secref{sec:exp}) use Barab\'{a}si--Albert (BA) graphs. To check whether our results are specific to that graph model, we repeat the node-scaling experiment on Erd\H{o}s--R\'{e}nyi (ER) graphs at matched density, otherwise using the identical setup (\texttt{m4} architecture, $k_{\max}=|V|$, 3 graphs $\times$ 3 contexts per configuration).

\begin{figure}[h]
\centering
\resizebox{1.0\columnwidth}{!}{\includegraphics{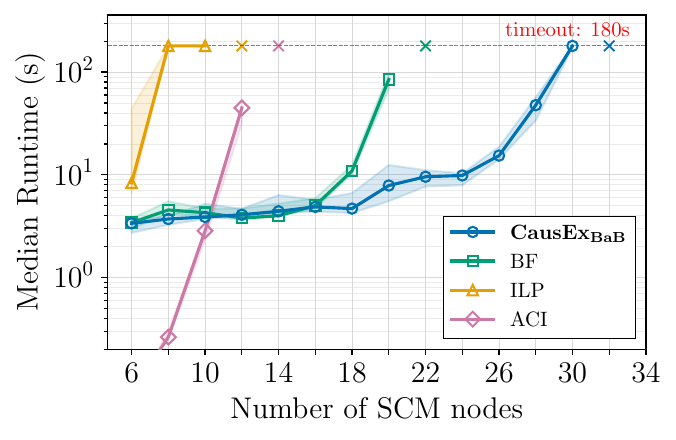}}
\caption{Runtime vs.\ number of SCM nodes on \emph{ER} graphs, for all methods. Cf.\ \figref{fig:nodes} for the matched BA-graph result.}
\label{fig:nodes_er}
\end{figure}

The relative ordering of methods and the overall shape of each curve matches the BA result (\figref{fig:nodes_er}): \causexbab times out on zero runs through $n=28$ on both graph types, \brutef's timeout rate jumps from 0\% to 100\% between $n=20$ and $n=22$ on both, and \ilp is already unusable by $n\geq10$--$12$ on both. The main quantitative difference is at the largest graph sizes: \causexbab's own timeout onset comes slightly earlier and steeper on ER than on BA (66.7\% timeout at $n=30$ and 100\% at $n=32$, vs.\ 33\% and 88.9\% at the same sizes on BA), and \aci's collapse likewise comes about two nodes earlier (100\% timeout already at $n=14$ on ER, vs.\ $n=16$ on BA). \ilp is the only method that fares slightly \emph{better} on ER at the smallest sizes (0\% timeout at $n=6$, vs.\ 11\% on BA). Full per-configuration results are in \tabref{tab:nodes_er}.

\begin{table*}[t]
\centering
\caption{Median runtime (s) vs.\ number of SCM nodes $n$ on \emph{ER} graphs, computing all minimal causes ($k_{\max}=|V|$), \texttt{m4} architecture (Fig.~\ref{fig:nodes_er}).}
\label{tab:nodes_er}
\setbox0=\hbox{\begin{tabular}{lrrrrrrrrrrrrrrr}
\toprule
$n$ & 6 & 8 & 10 & 12 & 14 & 16 & 18 & 20 & 22 & 24 & 26 & 28 & 30 & 32 & 34 \\
\midrule
\causexbab & 3.4 & 3.7 & 3.9 & 4.1 & 4.4 & 4.9 & 4.7 & 7.9 & 9.6 & 9.9 & 15.4 & 47.6 & 180.0$^\dagger$ & 180.0$^\ddagger$ & 180.0$^\ddagger$ \\
\brutef & 3.4 & 4.5 & 4.3 & 3.8 & 4.0 & 5.0 & 10.9 & 84.7 & 180.0$^\ddagger$ & 180.0$^\ddagger$ & 180.0$^\ddagger$ & 180.0$^\ddagger$ & 180.0$^\ddagger$ & 180.0$^\ddagger$ & 180.0$^\ddagger$ \\
\ilp & 8.4 & 180.0$^\dagger$ & 180.0$^\dagger$ & 180.0$^\ddagger$ & 180.0$^\ddagger$ & 180.0$^\ddagger$ & 180.0$^\ddagger$ & 180.0$^\ddagger$ & 180.0$^\ddagger$ & 180.0$^\ddagger$ & 180.0$^\ddagger$ & 180.0$^\ddagger$ & 180.0$^\ddagger$ & 180.0$^\ddagger$ & 180.0$^\ddagger$ \\
\aci & 0.1 & 0.3 & 2.8 & 44.8 & 180.0$^\ddagger$ & 180.0$^\ddagger$ & 180.0$^\ddagger$ & 180.0$^\ddagger$ & 180.0$^\ddagger$ & 180.0$^\ddagger$ & 180.0$^\ddagger$ & 180.0$^\ddagger$ & 180.0$^\ddagger$ & 180.0$^\ddagger$ & 180.0$^\ddagger$ \\
\bottomrule
\end{tabular}}\ifdim\wd0>\textwidth
\resizebox{\textwidth}{!}{\box0}
\else
\begin{tabular*}{\textwidth}{@{\extracolsep{\fill}}lrrrrrrrrrrrrrrr@{}}
\toprule
$n$ & 6 & 8 & 10 & 12 & 14 & 16 & 18 & 20 & 22 & 24 & 26 & 28 & 30 & 32 & 34 \\
\midrule
\causexbab & 3.4 & 3.7 & 3.9 & 4.1 & 4.4 & 4.9 & 4.7 & 7.9 & 9.6 & 9.9 & 15.4 & 47.6 & 180.0$^\dagger$ & 180.0$^\ddagger$ & 180.0$^\ddagger$ \\
\brutef & 3.4 & 4.5 & 4.3 & 3.8 & 4.0 & 5.0 & 10.9 & 84.7 & 180.0$^\ddagger$ & 180.0$^\ddagger$ & 180.0$^\ddagger$ & 180.0$^\ddagger$ & 180.0$^\ddagger$ & 180.0$^\ddagger$ & 180.0$^\ddagger$ \\
\ilp & 8.4 & 180.0$^\dagger$ & 180.0$^\dagger$ & 180.0$^\ddagger$ & 180.0$^\ddagger$ & 180.0$^\ddagger$ & 180.0$^\ddagger$ & 180.0$^\ddagger$ & 180.0$^\ddagger$ & 180.0$^\ddagger$ & 180.0$^\ddagger$ & 180.0$^\ddagger$ & 180.0$^\ddagger$ & 180.0$^\ddagger$ & 180.0$^\ddagger$ \\
\aci & 0.1 & 0.3 & 2.8 & 44.8 & 180.0$^\ddagger$ & 180.0$^\ddagger$ & 180.0$^\ddagger$ & 180.0$^\ddagger$ & 180.0$^\ddagger$ & 180.0$^\ddagger$ & 180.0$^\ddagger$ & 180.0$^\ddagger$ & 180.0$^\ddagger$ & 180.0$^\ddagger$ & 180.0$^\ddagger$ \\
\bottomrule
\end{tabular*}
\fi
\vspace{2pt}
{\footnotesize $^\dagger$median over a batch that includes at least one timed-out run (capped at 180s) $^\ddagger$every run in the batch timed out\par}
\end{table*}

\subsection{Maximum Cause Size $k_{\max}$}
\label{app:maxk}

\begin{figure}[h]
\centering
\resizebox{1.0\columnwidth}{!}{\includegraphics{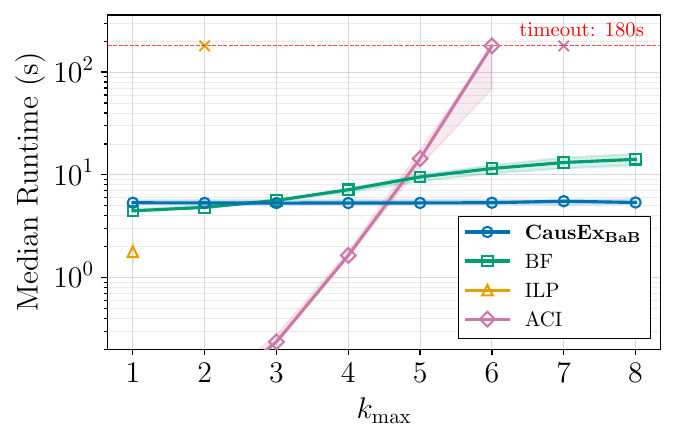}}
\caption{Runtime vs.\ maximum cause size $k_{\max}$ for all methods, $n=18$ nodes, \texttt{m4} architecture.}
\label{fig:maxk}
\end{figure}

We vary $k_{\max}$ from 1 to 8, fixing the graph size at 18 nodes and using the \texttt{m4} architecture (\figref{fig:maxk}). \causexbab's median runtime is essentially flat at 5.3--5.5s across all values of $k_{\max}$, while \brutef grows steadily from 4.5s at $k_{\max}=1$ to 14.2s at $k_{\max}=8$ (though it never times out at this graph size). \ilp solves $k_{\max}=1$ (median 1.8s) but times out on every run for $k_{\max}\geq 2$. \aci completes all runs through $k_{\max}=5$ (median growing from 0.05s to 14.4s) but times out on the majority of runs from $k_{\max}=6$ and on all runs by $k_{\max}=7$. Full per-configuration results are in \tabref{tab:maxk}.

\begin{table}[h]
\centering
\caption{Median runtime (s) vs.\ maximum cause size $k_{\max}$, $n=18$ nodes, \texttt{m4} architecture (Fig.~\ref{fig:maxk}).}
\label{tab:maxk}
\setbox0=\hbox{\begin{tabular}{lrrrrrrrr}
\toprule
$k_{\max}$ & 1 & 2 & 3 & 4 & 5 & 6 & 7 & 8 \\
\midrule
\causexbab & 5.3 & 5.3 & 5.3 & 5.3 & 5.3 & 5.3 & 5.5 & 5.4 \\
\brutef & 4.5 & 4.8 & 5.6 & 7.2 & 9.5 & 11.5 & 13.2 & 14.1 \\
\ilp & 1.8 & 180.0$^\ddagger$ & 180.0$^\ddagger$ & 180.0$^\ddagger$ & 180.0$^\ddagger$ & 180.0$^\ddagger$ & 180.0$^\ddagger$ & 180.0$^\ddagger$ \\
\aci & 0.0 & 0.1 & 0.2 & 1.6 & 14.4 & 180.0$^\dagger$ & 180.0$^\ddagger$ & 180.0$^\ddagger$ \\
\bottomrule
\end{tabular}}\ifdim\wd0>\columnwidth
\resizebox{\columnwidth}{!}{\box0}
\else
\begin{tabular*}{\columnwidth}{@{\extracolsep{\fill}}lrrrrrrrr@{}}
\toprule
$k_{\max}$ & 1 & 2 & 3 & 4 & 5 & 6 & 7 & 8 \\
\midrule
\causexbab & 5.3 & 5.3 & 5.3 & 5.3 & 5.3 & 5.3 & 5.5 & 5.4 \\
\brutef & 4.5 & 4.8 & 5.6 & 7.2 & 9.5 & 11.5 & 13.2 & 14.1 \\
\ilp & 1.8 & 180.0$^\ddagger$ & 180.0$^\ddagger$ & 180.0$^\ddagger$ & 180.0$^\ddagger$ & 180.0$^\ddagger$ & 180.0$^\ddagger$ & 180.0$^\ddagger$ \\
\aci & 0.0 & 0.1 & 0.2 & 1.6 & 14.4 & 180.0$^\dagger$ & 180.0$^\ddagger$ & 180.0$^\ddagger$ \\
\bottomrule
\end{tabular*}
\fi
\vspace{2pt}
{\footnotesize $^\dagger$median over a batch that includes at least one timed-out run (capped at 180s) $^\ddagger$every run in the batch timed out\par}
\end{table}

\subsection{Scalability}
\label{app:scalability}

\begin{figure}[h]
\centering
\resizebox{1.0\columnwidth}{!}{\includegraphics{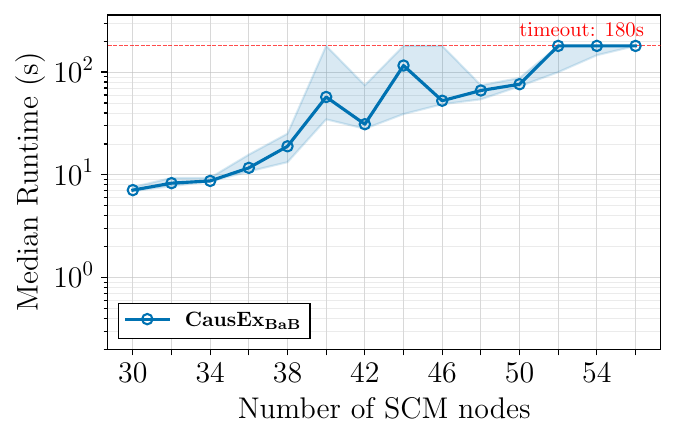}}
\caption{\causexbab runtime vs.\ number of SCM nodes on large graphs with $k_{\max}=8$ and the \texttt{m4} architecture. Competing methods timed out consistently at this scale and are omitted.}
    \label{fig:nodes_cbab}
\end{figure}

\begin{figure}[h]
\centering
\resizebox{1.0\columnwidth}{!}{\includegraphics{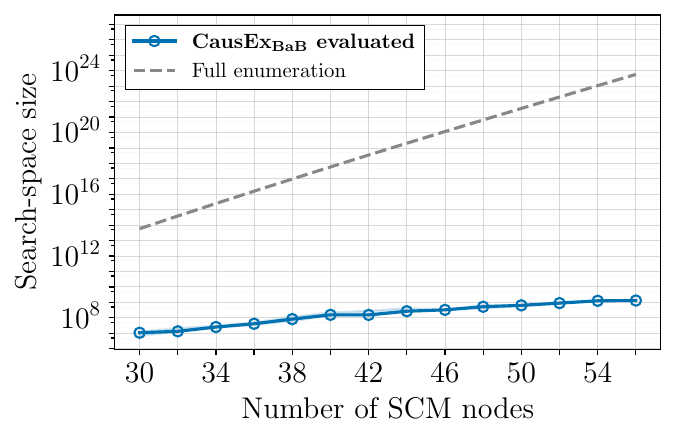}}
\caption{\causexbab nodes evaluated versus the theoretical full-enumeration bound, for $k_{\max}=8$ and the \texttt{m4} architecture.}
\label{fig:searchspace}
\end{figure}

Because the competing methods time out consistently beyond 20 nodes, we evaluate \causexbab alone, fixing $k_{\max}=8$, on graphs with up to 56 nodes (\figref{fig:nodes_cbab}, \tabref{tab:scaleability}). \causexbab completes almost every run up to $n = 38$, with a single exception at $n = 32$ (11\% timeout). Beyond $n = 40$ the timeout rate rises, though not monotonically (44\% at both $n = 44$ and $n = 46$, but 0\% at $n = 48$), and even at the largest tested size, $n = 56$, 11\% of runs still complete within the 180s budget. For reference, a graph with $n = 56$ intervenable nodes and $k_{\max} = 8$ has a worst-case search space of $\sum_{k=1}^{8}\binom{56}{k}\cdot 2^{56-k} \approx 5.8\times10^{23}$ candidate (cause set, contingency) pairs. \figref{fig:searchspace} plots the number of search-tree nodes \causexbab actually evaluates against this bound. The number evaluated stays many orders of magnitude below it, and the gap widens further as $n$ grows, i.e.\ pruning eliminates an increasing fraction of the theoretical search space at larger graph sizes.

\begin{table*}[t]
\centering
\caption{\causexbab median runtime (s) on large graphs, $k_{\max}=8$, \texttt{m4} architecture; competing methods time out throughout this range (Fig.~\ref{fig:nodes_cbab}).}
\label{tab:scaleability}
\setbox0=\hbox{\begin{tabular}{lrrrrrrrrrrrrrr}
\toprule
$n$ & 30 & 32 & 34 & 36 & 38 & 40 & 42 & 44 & 46 & 48 & 50 & 52 & 54 & 56 \\
\midrule
\causexbab & 7.1 & 8.3$^\dagger$ & 8.7 & 11.7 & 18.9 & 57.3$^\dagger$ & 31.1$^\dagger$ & 115.9$^\dagger$ & 52.6$^\dagger$ & 66.2 & 76.2$^\dagger$ & 177.4$^\dagger$ & 180.0$^\dagger$ & 180.0$^\dagger$ \\
\bottomrule
\end{tabular}}\ifdim\wd0>\textwidth
\resizebox{\textwidth}{!}{\box0}
\else
\begin{tabular*}{\textwidth}{@{\extracolsep{\fill}}lrrrrrrrrrrrrrr@{}}
\toprule
$n$ & 30 & 32 & 34 & 36 & 38 & 40 & 42 & 44 & 46 & 48 & 50 & 52 & 54 & 56 \\
\midrule
\causexbab & 7.1 & 8.3$^\dagger$ & 8.7 & 11.7 & 18.9 & 57.3$^\dagger$ & 31.1$^\dagger$ & 115.9$^\dagger$ & 52.6$^\dagger$ & 66.2 & 76.2$^\dagger$ & 177.4$^\dagger$ & 180.0$^\dagger$ & 180.0$^\dagger$ \\
\bottomrule
\end{tabular*}
\fi
\vspace{2pt}
{\footnotesize $^\dagger$median over a batch that includes at least one timed-out run (capped at 180s)\par}
\end{table*}

\subsection{Connected Components}
\label{app:components}
By construction, the BA preferential-attachment model almost surely produces a single connected component, since every new node must attach to at least $m$ existing nodes. To test robustness beyond this, we artificially introduce structural disconnection by partitioning the $n$ nodes into $c$ equally-sized subsets and sampling an independent BA DAG within each subset, yielding a graph whose components share no edges. The neural network still receives all $n$ endogenous variables as input, the disconnection affects only the causal structure, not the \ac{nn}'s feature set. We evaluate graphs of $n \in \{20, 22, 24, 26, 28\}$ nodes with $c \in \{1, 2, 3, 4, 5\}$ components, computing all minimal causes ($k_{\max} = |V|$) with the \texttt{m4} architecture.

\begin{figure}[h]
\centering
\resizebox{1.0\columnwidth}{!}{\includegraphics{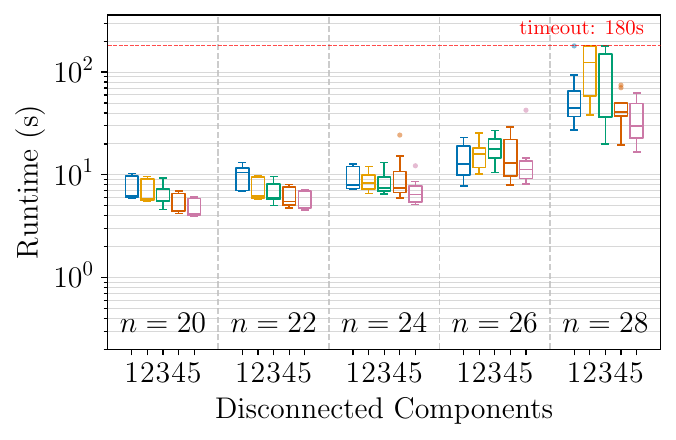}}
\caption{Runtime of \causexbab with varying numbers of connected components and SCM size (BA graphs, $m=6$).}
    \label{fig:components}
\end{figure}

\figref{fig:components} shows \causexbab runtimes for SCMs with 1 to 5 connected components across five graph sizes. The dominant factor is the number of nodes, with runtime growing steeply with $n$. The effect of the number of components is smaller and noisier. At $n \leq 24$, more components tend to reduce runtime (e.g.\ at $n=20$, median drops from 6.2s at $c=1$ to 4.2s at $c=5$), but the trend is not monotone at every $n$ and can invert, as seen at $n=28$, where $c=2$ is the slowest configuration (median 123.9s, 3 of 9 runs timing out) despite $c=1$ (44.6s, also affected by a timeout) and $c=5$ (30.0s, with no timeouts) both being faster.

The main driver is the number of valid causes, which grows substantially and unevenly with more components (e.g.\ at $n=28$ the mean cause count ranges from 2{,}203 at $c=5$ to 11{,}118 at $c=2$), so the algorithm must enumerate a much larger and less predictable solution set as the number of components grows, rather than pruning becoming steadily less effective. Pruning efficiency itself changes comparatively little and non-monotonically across $c$ at any fixed $n$. Full per-configuration results are in \tabref{tab:components}.

\begin{table}[h]
\centering
\caption{\causexbab median runtime (s) vs.\ number of components (rows) and SCM size $n$ (columns), computing all minimal causes ($k_{\max}=|V|$), \texttt{m4} architecture (Fig.~\ref{fig:components}).}
\label{tab:components}
\setbox0=\hbox{\begin{tabular}{lrrrrr}
\toprule
$c$ & 20 & 22 & 24 & 26 & 28 \\
\midrule
1 & 6.2 & 10.5 & 7.9 & 12.8 & 44.6$^\dagger$ \\
2 & 5.8 & 6.2 & 8.3 & 15.9 & 123.9$^\dagger$ \\
3 & 5.6 & 6.0 & 7.4 & 17.9 & 36.9$^\dagger$ \\
4 & 4.5 & 5.5 & 7.5 & 13.0 & 41.1 \\
5 & 4.2 & 4.8 & 6.4 & 11.2 & 30.0 \\
\bottomrule
\end{tabular}}\ifdim\wd0>\columnwidth
\resizebox{\columnwidth}{!}{\box0}
\else
\begin{tabular*}{\columnwidth}{@{\extracolsep{\fill}}lrrrrr@{}}
\toprule
$c$ & 20 & 22 & 24 & 26 & 28 \\
\midrule
1 & 6.2 & 10.5 & 7.9 & 12.8 & 44.6$^\dagger$ \\
2 & 5.8 & 6.2 & 8.3 & 15.9 & 123.9$^\dagger$ \\
3 & 5.6 & 6.0 & 7.4 & 17.9 & 36.9$^\dagger$ \\
4 & 4.5 & 5.5 & 7.5 & 13.0 & 41.1 \\
5 & 4.2 & 4.8 & 6.4 & 11.2 & 30.0 \\
\bottomrule
\end{tabular*}
\fi
\vspace{2pt}
{\footnotesize $^\dagger$median over a batch that includes at least one timed-out run (capped at 180s)\par}
\end{table}

\subsection{Sparsity}
\label{app:sparsity}
In the BA model, the attachment parameter $m$ controls how many edges each new node forms, directly determining graph density. At $m=1$ the graph reduces to a tree, while larger $m$ produces progressively denser graphs with higher average in-degree. We fix $n=28$ nodes and vary $m \in \{1,\ldots,8\}$ to test how edge density affects runtime, computing all minimal causes ($k_{\max}=|V|$) with the \texttt{m4} architecture constant.

\begin{figure}[h]
\centering
\resizebox{1.0\columnwidth}{!}{\includegraphics{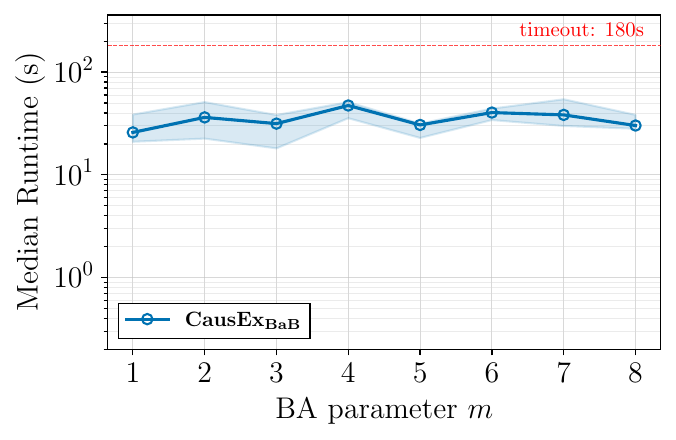}}
\caption{Runtime of \causexbab on BA graphs with $n=28$ nodes, varying attachment parameter $m$.}
    \label{fig:sparsity}
\end{figure}

\figref{fig:sparsity} shows no clear relationship between edge density and runtime. Median runtime stays within a 26--47s band across all tested values of $m$, with no monotone trend, and variance is high at every value. Only 4 of 72 runs time out, all at $m=4$ or $m=7$. Overall, \causexbab is robust to edge density at this graph size. Full per-configuration results are in \tabref{tab:sparsity}.

\begin{table}[h]
\centering
\caption{\causexbab median runtime (s) vs.\ BA attachment parameter $m$, $n=28$ nodes, computing all minimal causes ($k_{\max}=|V|$) (Fig.~\ref{fig:sparsity}).}
\label{tab:sparsity}
\setbox0=\hbox{\begin{tabular}{lrrrrrrrr}
\toprule
$m$ & 1 & 2 & 3 & 4 & 5 & 6 & 7 & 8 \\
\midrule
\causexbab & 25.9 & 36.2 & 31.5 & 47.3$^\dagger$ & 30.6 & 40.4 & 38.3$^\dagger$ & 30.2 \\
\bottomrule
\end{tabular}}\ifdim\wd0>\columnwidth
\resizebox{\columnwidth}{!}{\box0}
\else
\begin{tabular*}{\columnwidth}{@{\extracolsep{\fill}}lrrrrrrrr@{}}
\toprule
$m$ & 1 & 2 & 3 & 4 & 5 & 6 & 7 & 8 \\
\midrule
\causexbab & 25.9 & 36.2 & 31.5 & 47.3$^\dagger$ & 30.6 & 40.4 & 38.3$^\dagger$ & 30.2 \\
\bottomrule
\end{tabular*}
\fi
\vspace{2pt}
{\footnotesize $^\dagger$median over a batch that includes at least one timed-out run (capped at 180s)\par}
\end{table}

\subsection{Cause Sizes}
\label{app:causesizes}

\begin{figure}[h]
\centering
\resizebox{1.0\columnwidth}{!}{\includegraphics{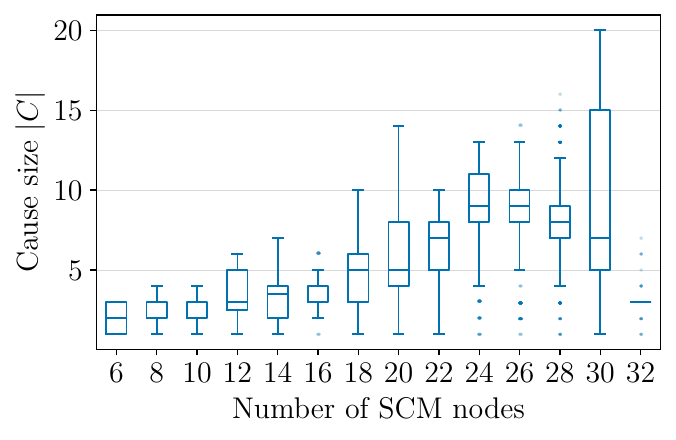}}
\caption{Distribution of $|C|$ over minimal actual causes found, vs.\ number of SCM nodes $V$, for \causexbab with $k_{\max} = |V|$. Boxes show the median and interquartile range over completed instances at each graph size (timeouts and OOM runs excluded), whiskers extend to $1.5\times$ IQR, outliers omitted.}
\label{fig:causesizes}
\end{figure}

\figref{fig:causesizes} shows the distribution of minimal cause sizes $|C|$ found by \causexbab across completed instances of the node-scaling benchmark ($k_{\max}=|V|$). Sizes are broadly spread and not concentrated near the minimum. \tabref{tab:cause_size_dist} gives the cumulative share of causes captured as the size cap $k_{\max}$ increases. Capping the search at $k_{\max}=8$ retains only $59.4\%$ of all minimal causes found by unrestricted search, and even $k_{\max}=15$ still misses $5.3\%$. \tabref{tab:cause_counts} instead reports, per instance, how many distinct minimal causes are found. $18.3\%$ of instances have none, $13.0\%$ have exactly one, and the remaining $68.7\%$ have two or more, with $29.6\%$ having at least 15.

\begin{table}[h]
\centering
\caption{Cumulative share of the 10885 minimal causes found with size $|C| \leq k_{\max}$ (i.e.\ the fraction \causexbab would still capture if capped at that $k_{\max}$; not to be confused with the number of \emph{distinct} causes per instance in Table~\ref{tab:cause_counts}), over the same completed \causexbab runs on graphs of $n=6$ to $32$ nodes with $k_{\max}=|V|$.}
\label{tab:cause_size_dist}
\setbox0=\hbox{\begin{tabular}{lrr}
\toprule
$k_{\max}$ & $n$ & \% \\
\midrule
1 & 52 & 0.5 \\
2 & 181 & 1.7 \\
5 & 2012 & 18.5 \\
8 & 6466 & 59.4 \\
10 & 9079 & 83.4 \\
15 & 10307 & 94.7 \\
\bottomrule
\end{tabular}}\ifdim\wd0>\columnwidth
\resizebox{\columnwidth}{!}{\box0}
\else
\begin{tabular*}{\columnwidth}{@{\extracolsep{\fill}}lrr@{}}
\toprule
$k_{\max}$ & $n$ & \% \\
\midrule
1 & 52 & 0.5 \\
2 & 181 & 1.7 \\
5 & 2012 & 18.5 \\
8 & 6466 & 59.4 \\
10 & 9079 & 83.4 \\
15 & 10307 & 94.7 \\
\bottomrule
\end{tabular*}
\fi
\end{table}

\begin{table}[h]
\centering
\caption{Number of \emph{distinct} minimal causes found per instance (not to be confused with the cause \emph{size} $|C|$ in Fig.~\ref{fig:causesizes}), over 115 completed \causexbab runs on graphs of $n=6$ to $32$ nodes with $k_{\max}=|V|$.}
\label{tab:cause_counts}
\setbox0=\hbox{\begin{tabular}{lrr}
\toprule
\# causes & $n$ & \% \\
\midrule
0 & 21 & 18.3 \\
1 & 15 & 13.0 \\
$\geq$2 & 79 & 68.7 \\
$\geq$5 & 59 & 51.3 \\
$\geq$10 & 42 & 36.5 \\
$\geq$15 & 34 & 29.6 \\
\bottomrule
\end{tabular}}\ifdim\wd0>\columnwidth
\resizebox{\columnwidth}{!}{\box0}
\else
\begin{tabular*}{\columnwidth}{@{\extracolsep{\fill}}lrr@{}}
\toprule
\# causes & $n$ & \% \\
\midrule
0 & 21 & 18.3 \\
1 & 15 & 13.0 \\
$\geq$2 & 79 & 68.7 \\
$\geq$5 & 59 & 51.3 \\
$\geq$10 & 42 & 36.5 \\
$\geq$15 & 34 & 29.6 \\
\bottomrule
\end{tabular*}
\fi
\end{table}

\subsection{Branch Selection}
\label{app:branchselection}

\begin{figure}[h]
\centering
\resizebox{1.0\columnwidth}{!}{\includegraphics{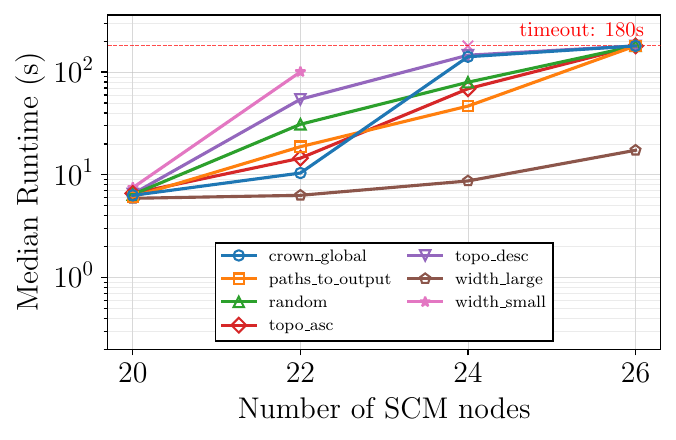}}
\caption{Runtime vs.\ number of SCM nodes for seven branch-selection strategies. \texttt{width\_large} is used as default in all other experiments.}
\label{fig:heuristic_nodes}
\end{figure}

\begin{figure}[h]
\centering
\resizebox{1.0\columnwidth}{!}{\includegraphics{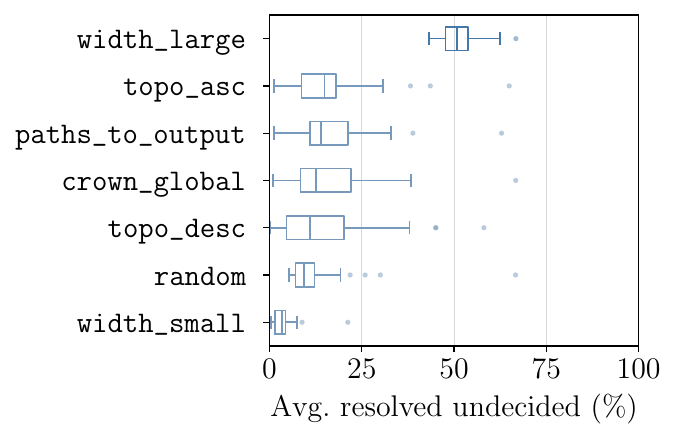}}
\caption{Pruning efficiency (\% undecided nodes resolved per branch) of the same seven strategies, mean $\pm$ std over all graph sizes and contexts.}
\label{fig:heuristic_prune}
\end{figure}

We evaluate seven branching strategies (see \algref{alg:causexbab}, $\mathrm{selectVar}$) on our method alone, using graphs of
20 to 26 nodes in increments of 2, computing all minimal causes ($k_{\max} = |V|$), and the \texttt{m4}
architecture. The strategies differ in which variable to split on next:

\begin{itemize}
    \item \texttt{topo\_asc} / \texttt{topo\_desc.}
          Selects the undecided variable with the lowest (asc) or highest (desc)
          topological index, branching from causes toward effects or vice versa.

    \item \texttt{width\_large} / \texttt{width\_small.}
          Selects the variable with the largest (large) or smallest (small)
          current interval width, prioritising either the most uncertain or the
          most nearly resolved dimension.

    \item \texttt{paths\_to\_output.}
          Selects the variable with the most directed paths to the output node,
          prioritising variables with the highest structural influence on the prediction.

    \item \texttt{crown\_global.}
          Uses CROWN sensitivity scores computed once at the root to rank
          variables by their influence on the output.

    \item \texttt{random.}
          Selects the next variable to split on uniformly at random among the
          undecided variables, serving as a baseline.
\end{itemize}

\texttt{width\_large}, which always selects the undecided variable with the widest current interval, clearly outperforms all alternatives (\figref{fig:heuristic_nodes}). At the largest tested graph size ($n = 26$), all competing strategies time out on the majority (67--100\%) of runs, while \texttt{width\_large} completes without a single timeout. \figref{fig:heuristic_prune} shows the pruning efficiency of each strategy. \texttt{width\_large} resolves on average 51.2\% of undecided nodes per pruned branch, compared to 3.7--16.9\% for the other strategies. That is, it prunes branches much earlier in the search tree, avoiding most of the exponential blow-up. We therefore use \texttt{width\_large} as the default in all other experiments. \tabref{tab:heuristics} gives the exact mean, standard deviation, and median pruning efficiency for every strategy.

\begin{table}[h]
\centering
\caption{Avg.\ \% of undecided nodes resolved per branch, pooled over all graph sizes and contexts (Fig.~\ref{fig:heuristic_prune}).}
\label{tab:heuristics}
\setbox0=\hbox{\begin{tabular}{lrrr}
\toprule
Strategy & Mean & Std & Median \\
\midrule
\texttt{width\_large} & 51.2 & 5.2 & 50.8 \\
\texttt{crown\_global} & 16.9 & 13.3 & 12.7 \\
\texttt{paths\_to\_output} & 16.9 & 11.9 & 14.0 \\
\texttt{random} & 12.3 & 10.8 & 9.4 \\
\texttt{topo\_asc} & 15.9 & 12.6 & 14.9 \\
\texttt{topo\_desc} & 15.7 & 14.4 & 10.9 \\
\texttt{width\_small} & 3.7 & 3.6 & 3.3 \\
\bottomrule
\end{tabular}}\ifdim\wd0>\columnwidth
\resizebox{\columnwidth}{!}{\box0}
\else
\begin{tabular*}{\columnwidth}{@{\extracolsep{\fill}}lrrr@{}}
\toprule
Strategy & Mean & Std & Median \\
\midrule
\texttt{width\_large} & 51.2 & 5.2 & 50.8 \\
\texttt{crown\_global} & 16.9 & 13.3 & 12.7 \\
\texttt{paths\_to\_output} & 16.9 & 11.9 & 14.0 \\
\texttt{random} & 12.3 & 10.8 & 9.4 \\
\texttt{topo\_asc} & 15.9 & 12.6 & 14.9 \\
\texttt{topo\_desc} & 15.7 & 14.4 & 10.9 \\
\texttt{width\_small} & 3.7 & 3.6 & 3.3 \\
\bottomrule
\end{tabular*}
\fi
\end{table}

\subsection{Feature Independence}
\label{app:independence}

We run \causexbab on the node-scaling benchmark (BA graphs, $n=6$--$34$, \texttt{m4}, $k_{\max}=|V|$) a second time on each \ac{scm} with all edges between endogenous variables removed (\texttt{causexbab\_indep}), then re-check every resulting cause against the true structural equations. Excluding instances where either method timed out ($n\geq30$), 111 of 135 remain.

Independence finds $59\%$ more causes that are $20\%$ larger on average, yet recovers only $33\%$ of the true minimal causes. \tabref{tab:indep_breakdown} shows why: three quarters of what it finds is a valid intervention but not minimal, because respecting the \ac{scm} would let a smaller cause reach the same effect by propagating to other variables.

\begin{table}[h]
\centering
\small
\begin{tabular}{lr}
\toprule
Independence causes & \% \\
\midrule
Exactly match an \ac{scm} cause & 20.7 \\
Valid but not minimal under \ac{scm} & 74.3 \\
Spurious under \ac{scm} & 5.0 \\
\bottomrule
\end{tabular}
\caption{Composition of causes found under the independence assumption, checked against the true \ac{scm} (pooled over 111 instances, 15{,}776 causes).}
\label{tab:indep_breakdown}
\end{table}

\section{Case Study: SNAP Quality-Control Review Targeting}
\label{app:casestudy}

The Supplemental Nutrition Assistance Program \cite{Leftin2026} is the largest food-assistance program in the United States. USDA runs Quality Control in order to measure payment error rates, and states carry financial liability for those rates, facing penalties when they stay high and receiving bonuses when they stay low. A review is a full re-investigation of a case and is correspondingly expensive, so only a small sample is reviewed each year. The natural use of a predictive model is therefore review targeting: given a certified household's eligibility profile, predict how likely it is that the original determination was wrong, and use that to decide which cases to review.

\paragraph{SCM.}
Who qualifies for SNAP is codified in the federal regulations governing the program, 7~CFR~Part~273 \cite{CFR273}, as an explicit Boolean combination of an income test, an asset test, and categorical-eligibility pathways. The \ac{scm} formalizes this statutory definition directly.
Every gate traces to a specific, numbered subsection of this regulation, obtained by reading the text and translating each test directly into a node: for example, 273.9(a) grants an income-test exemption to households that contain ``an elderly or disabled member,'' which resolves directly to \texttt{has\_elderly\_or\_disabled} = OR(\texttt{has\_elderly}, \texttt{has\_disabled}).

The \ac{scm} has 20 binary nodes (\tabref{tab:snap_nodes}): 10 raw leaves, each testing a single Boolean condition (e.g.\ whether all household members receive SSI, or whether gross income is below the threshold) computed from a group of the file's 65 raw columns (per-person benefit amounts, household size, income totals) via a threshold or membership test, and 10 derived nodes combining them with AND, OR, and NOT gates. Ten of the 20 features are therefore deterministic functions of the others.

\paragraph{Dataset.}
We use USDA's FY2024 SNAP Quality Control public-use file \cite{USDAFNS2026}, keeping only the columns needed to determine SNAP eligibility under the \ac{scm}, together with the fields needed to select households and construct the label. The raw file has $44{,}891$ households. After keeping only complete, determinately-labeled records in contiguous US/Guam/Virgin Islands states, $43{,}771$ remain, split $70/15/15$ into train/val/test ($30{,}639$/$6{,}565$/$6{,}567$).

\paragraph{Neural Network.}
We train an NN with hidden dimensions $[256\times128]$, ReLU activations, and a sigmoid output. The model takes a household's eligibility profile and predicts whether a re-investigation would find an error. \secref{sec:background}'s formal treatment of \acp{nn} composes only affine and ReLU layers. Bounding this additional sigmoid output layer is a direct extension, since sigmoid is monotone and its bounds follow immediately from the bounds on its input. Error prevalence is close to $40\%$ in every split, already well balanced. On the held-out test split, the NN reaches AUROC $0.622$ and accuracy $0.628$ (majority-class baseline $60.3\%$), comparable to a logistic-regression baseline on the same features (AUROC $0.618$, accuracy $0.620$): both find only a modest signal, expected for a label that records whether a human caseworker's original determination was later found wrong rather than a deterministic function of the household's eligibility profile. This does not affect our evaluation of \causexbab itself: the runtime comparison across methods and the comparison between treating inputs as independent versus respecting the \ac{scm} are both properties of the causal search over a fixed decision function, not of how well that function predicts the review outcome.

The regulation is used to compute the 10 derived nodes of the input space, not just the 10 raw leaves, so the network has access to the same derived concepts (e.g.\ \texttt{categorically\_eligible}, \texttt{assets\_test\_passed}) a caseworker reasons with. The label, recorded in the dataset as \texttt{STATUS}, is a QC reviewer's finding on a separate, later determination of whether the original certification was correct.

\paragraph{Computing Causes.}
We flag a household for review when the \ac{nn}'s output crosses the standard $0.5$ threshold, so the question a cause should answer is: why was this household flagged for review? We draw the held-out test instances flagged by the \ac{nn}, deduplicated by eligibility profile so repeated households do not inflate the sample: since the \ac{scm} has only 20 binary nodes, many real households share an identical profile, and only 11 distinct profiles exist among flagged households in the test split, so we use all 11 rather than a subsample of them.
We run \causexbab, brute-force, \textsc{ILP}, and \textsc{ACI} on each, searching for witnesses that flip the prediction below $0.5$, with $k_{\max} = |V|$ and a 180s timeout per method per instance.

\causexbab completes all 11 instances within budget (median $4.60$s, $0/11$ timeouts), finding a median of $23$ minimal causes, $33.7\%$ of them singletons. \brutef also completes every instance (median $43.24$s). \textsc{ILP} times out on $3/11$, and \textsc{ACI} times out on $10/11$ and succeeds only once (\tabref{tab:snap_benchmark}). This reverses the ordering seen in \secref{sec:exp}, where \textsc{ACI} is the faster of the two at comparable scale. We conjecture this comes from the SNAP \ac{scm}'s structure being a poorer match for \textsc{ACI}'s heuristic search than the random benchmark graphs, while \textsc{ILP} benefits here from a simpler acceptance set $S$ (\remref{rem:satisfactionset}).

\begin{table}[t]
\centering
\footnotesize
\begin{tabular}{lrrr}
\toprule
Method & Runtime (s) & IQR & Timeout \\
\midrule
\causexbab & 4.6   & [4.6, 4.6]   & 0/11 \\
\brutef    & 43.2  & [17.8, 49.1] & 0/11 \\
\ilp       & 122.5 & [113.0, 160.6] & 3/11 \\
\aci       & 180.0 & [180.0, 180.0] & 10/11 \\
\bottomrule
\end{tabular}
\caption{Runtime of all four methods, 180s timeout per method per instance.}
\label{tab:snap_benchmark}
\end{table}

\paragraph{Causes with vs.\ without the \ac{scm}.}
To measure what treating the 20 inputs as independent costs in practice, we rerun \causexbab on a variant of the \ac{scm} with the 10 derived-node edges removed, so all 20 features are independent, and re-evaluate every resulting cause through the real structural equations, letting the derived nodes recompute from their true parents instead of staying frozen at their factual value, checking whether the resulting prediction still crosses the $0.5$ threshold.

Ignoring the dependency structure more than doubles the causes found (median $50$ vs.\ $23$ under the \ac{scm}, \tabref{tab:snap_structure_comparison}). Of the $1{,}118$ independent-model causes across all 11 instances, $14.9\%$ are spurious once downstream nodes are recomputed from their true parents, and $14.8\%$ coincide with a cause found under the \ac{scm} (\tabref{tab:snap_structure_spurious}).

\begin{table}[htb]
\centering
\footnotesize
\begin{tabular}{lrrr}
\toprule
& Median & IQR & Range \\
\midrule
Causes, \ac{scm}   & 23   & [20.5, 34.5]  & [18, 96] \\
Causes, indep.     & 50   & [36.0, 113.5] & [18, 475] \\
\bottomrule
\end{tabular}
\caption{Causes found treating the 20 SNAP features as independent (no \ac{scm} edges) vs.\ under the \ac{scm}, 180s timeout.}
\label{tab:snap_structure_comparison}
\end{table}

\begin{table}[!htb]
\centering
\small
\begin{tabular}{lr}
\toprule
Independent-model causes (1{,}118 total) & Count (\%) \\
\midrule
Spurious under the \ac{scm}      & 167 (14.9\%) \\
Exactly match an \ac{scm} cause  & 166 (14.8\%) \\
\bottomrule
\end{tabular}
\caption{``Spurious'' counts independent-model causes whose intervention no longer flips the prediction below $0.5$ once derived nodes are recomputed from their true parents.}
\label{tab:snap_structure_spurious}
\end{table}

\paragraph{Limitations.}
The \ac{scm} is a simplified version of the real rule: it leaves out deductions, state waivers, and special-household cases, so \texttt{eligible} can disagree with what a caseworker actually decided for a given household. The households in our sample were also already certified for benefits, not drawn from all applicants, so our results speak to error risk among certified households rather than eligibility determination in general. Derived nodes such as \texttt{eligible} are themselves valid intervention targets under our formalism, so a returned cause may name one directly, a counterfactual a caseworker could not realize by acting on that node alone rather than on its determinants. We make no claim that a returned cause is the uniquely correct or most actionable explanation, only that it is a provably minimal one under AC1--AC3. The 14.9\% spurious-cause figure is specific to this SCM and this predictor. A more complete SCM or a stronger predictor could yield a different number, so it should be read as evidence from this case study rather than a general estimate for SNAP QC targeting.

\begin{table*}[t]
\centering
\small
\begin{tabularx}{\textwidth}{ p{0.26\textwidth} | p{0.54\textwidth} | X }
\toprule
Node & Definition & Source \\
\midrule
\multicolumn{3}{l}{\textit{Raw leaves (10)}} \\
\texttt{ssi\_receipt}         & All members' SSI benefits $>0$  & 273.2(j)(2)(i)(D) \\
\texttt{tanf\_all\_members}   & All members' TANF benefits $>0$ & 273.2(j)(2)(i)(A) \\
\texttt{mn\_fip}              & Minnesota Family Investment Program flag & TechDoc p.63 \\
\texttt{ga\_receipt}          & All members' GA benefits $>0$   & 273.2(j)(4) \\
\texttt{has\_elderly}         & Elderly member in unit     & 273.9(a) \\
\texttt{has\_disabled}        & Disabled member in unit    & 273.9(a) \\
\texttt{gross\_income\_low}   & $\leq$ Table~F.1 screen    & 273.9(a)(1)(i) \\
\texttt{net\_income\_low}     & $\leq$ Table~F.2 screen    & 273.9(a)(2)(i) \\
\texttt{assets\_low\_standard} & Assets $\leq$ \$2{,}750   & 273.8(b)/(b)(1) \\
\texttt{assets\_low\_higher}  & Assets $\leq$ \$4{,}250    & 273.8(b)/(b)(1) \\
\midrule
\multicolumn{3}{l}{\textit{Derived (10)}} \\
\texttt{tanf\_receipt}        & OR(\texttt{tanf\_all\_members}, \texttt{mn\_fip}) & 273.2(j)(2)(i)(A) + TechDoc p.63 \\
\texttt{has\_elderly\_or\_disabled} & OR(\texttt{has\_elderly}, \texttt{has\_disabled}) & 273.9(a) \\
\texttt{not\_elderly\_or\_disabled} & NOT(\texttt{has\_elderly\_or\_disabled}) & 273.8(b) \\
\texttt{assets\_term\_eld}    & AND(\texttt{has\_elderly\_or\_disabled}, \texttt{assets\_low\_higher}) & 273.8(b) \\
\texttt{assets\_term\_std}    & AND(\texttt{not\_elderly\_or\_disabled}, \texttt{assets\_low\_standard}) & 273.8(b) \\
\texttt{assets\_test\_passed} & OR(\texttt{assets\_term\_eld}, \texttt{assets\_term\_std}) & 273.8(b) \\
\texttt{gross\_income\_test\_passed} & OR(\texttt{has\_elderly\_or\_disabled}, \texttt{gross\_income\_low}) & 273.9(a) \\
\texttt{categorically\_eligible} & OR(\texttt{ssi\_receipt}, \texttt{tanf\_receipt}, \texttt{ga\_receipt}) & 273.2(j)(2)(i)/(j)(4) \\
\texttt{income\_asset\_eligible} & AND(\texttt{gross\_income\_test\_passed}, \texttt{net\_income\_low}, \texttt{assets\_test\_passed}) & 273.9(a) + 273.8(a) \\
\texttt{eligible}             & OR(\texttt{categorically\_eligible}, \texttt{income\_asset\_eligible}) & 273.8(a)/273.9(a) \\
\bottomrule
\end{tabularx}
\caption{The 20 nodes of the SNAP \ac{scm}. Sources are sections of 7~CFR~273 \cite{CFR273}, with dollar thresholds from the FY2024 Quality Control Technical Documentation \cite{Leftin2026}, Appendix~F.}
\label{tab:snap_nodes}
\end{table*}

\end{document}